\documentclass[lettersize,journal]{IEEEtran}
\usepackage{amsmath,amsfonts}
\usepackage{algorithm}
\usepackage{array}
\usepackage{textcomp}
\usepackage{stfloats}
\usepackage{url}
\usepackage{verbatim}
\usepackage{graphicx}
\usepackage{cite}
\usepackage[utf8]{inputenc} 
\usepackage[T1]{fontenc}    

\usepackage{booktabs}       
\usepackage{amsfonts}       
\usepackage{nicefrac}       
\usepackage{microtype}      
\usepackage{xcolor}         
\usepackage{microtype}
\usepackage{graphicx}
\usepackage{subfigure}
\usepackage{booktabs} 
\usepackage{url}
\usepackage{color}
\usepackage{amsmath}
\usepackage{amssymb}
\usepackage{mathtools}
\usepackage{amsthm}
\usepackage{amsfonts,extarrows}
\usepackage{bm}
\usepackage{multirow}
\usepackage{subfigure} 

\usepackage{algorithm}
\usepackage{algorithmic}
\newtheorem{definition}{Definition} 
\newtheorem{theorem}{Theorem}
\newtheorem{assumption}{Assumption}
\newtheorem{lemma}{Lemma}
\newtheorem{proposition}{Proposition}

\begin{document}

\title{Mode Connectivity and Data Heterogeneity of Federated Learning}

\author{Tailin~Zhou, \IEEEmembership{Graduate Student Member,~IEEE,}
        Jun~Zhang,~\IEEEmembership{Fellow,~IEEE,}
        and~Danny~H.K.~Tsang,~\IEEEmembership{Life~Fellow,~IEEE}
\thanks{
 T. Zhou is with IPO,  Academy of Interdisciplinary Studies, The Hong Kong University of Science and Technology, Clear Water Bay, Hong Kong SAR, China (Email: tzhouaq@connect.ust.hk).
J. Zhang is with the Department of Electronic and Computer Engineering, The Hong Kong University of Science and Technology, Clear Water Bay, Hong Kong SAR, China  (E-mail: eejzhang@ust.hk).
D. H.K. Tsang is with the Internet of Things Thrust, The Hong Kong University of Science and Technology (Guangzhou), Guangzhou, Guangdong, China, and also with the Department of Electronic and Computer Engineering, The Hong Kong University of Science and Technology, Clear Water Bay, Hong Kong SAR, China (Email: eetsang@ust.hk). (The corresponding author is J. Zhang.)
}
}



\maketitle

\begin{abstract}
Federated learning (FL)  enables multiple clients to train a model while keeping their data private collaboratively.
Previous studies have shown that data heterogeneity between clients leads to drifts across client updates.
However, there are few studies on the relationship between client and global modes,  making it unclear where these updates end up drifting.
We perform empirical and theoretical studies on this relationship by utilizing mode connectivity, which measures performance change (i.e., connectivity) along parametric paths between different modes.
Empirically,   reducing data heterogeneity makes the connectivity on different paths more similar, forming more low-error overlaps between client and global modes.
We also find that a barrier to connectivity occurs when linearly connecting two global modes, while it disappears with considering non-linear mode connectivity.
Theoretically,   we establish a quantitative bound on the global-mode connectivity using mean-field theory or dropout stability.  
The bound demonstrates that the connectivity improves when reducing data heterogeneity and widening trained models.
Numerical results further corroborate our analytical findings. 
\end{abstract}


\section{Introduction}
\IEEEPARstart{F}{ederated} Learning (FL) \cite{mcmahan17FL} is a collaborative paradigm that enables multiple clients to train a model while preserving data privacy.
 A primary challenge in FL is data heterogeneity across clients \cite{kairouz2021advances}. 
Previous works like \cite{zhao2018federated} found that data heterogeneity induces a misalignment between clients' local objectives and the FL's global objective (hereinafter referred to as client objective and FL objective, respectively) and drifts client updates.
Nevertheless, despite client-update drift, the FL objective can typically be optimized \cite{shao2023survey} effectively.
This implies that the destinations of these drifting updates (i.e., the solutions to the client and FL objectives)  may overlap.
However, a systematic investigation of this relationship is currently lacking.

A straightforward method to examine this overlapping relationship is to investigate mode connectivity \cite{Garipov2018Loss_Surfaces}, where mode refers to a suit of solutions to a specific objective,  i.e., a suit of neuron parameters, under permutation invariance.
Mode connectivity measures the connectivity of two modes via a parametric path and reveals the geometric relationship in the solution space.
Specifically, it evaluates the performance change (e.g., taking loss/error as metrics)  along a given path with endpoints corresponding to two different modes.
Maintaining a minor change along the path indicates better connectivity (i.e., the geometric landscape becomes increasingly connected between the two modes).
Some works \cite{Garipov2018Loss_Surfaces, Draxler2018No_Barriers} have discovered that two modes can be connected well in centralized training, but research on FL remains scarce.
 
 \begin{figure}[t]
     \centering
     \includegraphics[width=0.4\textwidth]{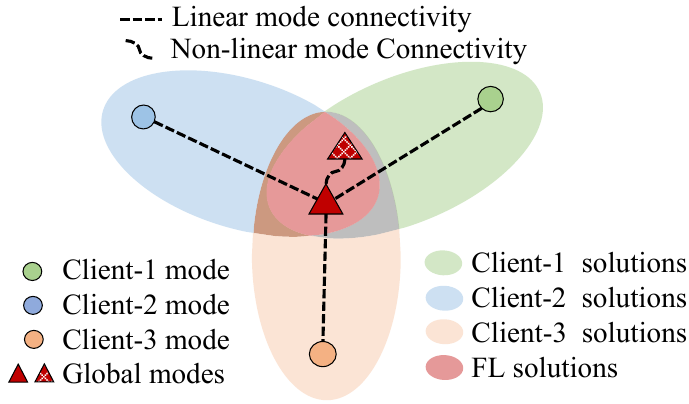}
     \caption{An illustration to show the relationship of client modes and global modes of FL in the solution space.}
     \label{fig:toy_example_mode_space}
 \end{figure}
 
In this work, we conduct empirical and theoretical studies on mode connectivity in FL under data heterogeneity to unravel the relationship between client modes (i.e., the solutions to client objectives) and global modes (i.e., the solutions to the FL objective).
Empirically, we consider linear and non-linear paths to evaluate the connectivity between client modes and global modes to show their overlap relationships, as illustrated in Figure \ref{fig:toy_example_mode_space}.
We also investigate the connectivity of global modes obtained under varying levels of data heterogeneity.
Theoretically,  we leverage mean-field theory \cite{SongMei2018meanfield} and dropout stability to establish a quantitative bound on the connectivity errors of global modes. 
In addition, we conduct a comprehensive numerical analysis to verify our analytical findings.

This study takes the first step to explore mode connectivity in FL.
Our main contributions are summarized as follows:
\begin{itemize}
    \item Our empirical and theoretical investigations demonstrate that data heterogeneity deteriorates the connectivity between client and global modes, making it challenging  to find solutions that meet both client and FL objectives.

    \item It is revealed that FL solutions from different data heterogeneity belong to distinct global modes within a common region of the solution space.
    They can be connected on a polygonal chain while keeping error low but face an error barrier when considering a linear path. 
    
    \item We establish quantitative bounds on dropout and connectivity errors of global modes using mean-field theory, showing that both errors decrease with reduced data heterogeneity and widening trained models.
\end{itemize}


\section{Related Works}

 \textbf{Data heterogeneity in FL.}
According to \cite{zhao2018federated}, data heterogeneity across clients significantly degrades the FL performance. Subsequent works have observed that data heterogeneity induces weight divergence and feature inconsistency.
For weight divergence, studies such as \cite{li2020federated,wang2020tackling} show that client-update drift slows the FL convergence due to inconsistent client and FL objectives.
Moreover, some studies,  e.g., \cite{luo2021no,zhou2022fedfa}, discover that the classifier divergence across clients is the main culprit of feature inconsistency, negatively affecting FL generalization.
In contrast to previous works, we offer a novel perspective on data heterogeneity's impact within the FL solution space.

 \textbf{Mode connectivity.}
The connectivity between neural network (NN) solutions due to the NN non-convexity, called mode connectivity, provides a new view into the geometric landscape of NNs in recent years.
Despite the high dimension of NNs, the different solutions follow a clear connectivity pattern \cite{FreemanB17, garipov2018loss,Draxler2018No_Barriers}.
Specifically, it is possible to connect the local minima through specific paths that make it easier to find high-performance solutions when using gradient-based optimizers.
The paths are often non-linear curves and are discovered by specific tasks. 
This suggests that the local minima are not isolated but rather interconnected within a manifold.
Moreover,  the geometric landscape of an NN becomes approximately linear-connected when considering the infinite neuron numbers as per \cite{FreemanB17}.
Previous studies mainly focus on mode connectivity in centralized training with homogeneous data.
Several pilot studies \cite{li2022understanding,zhou2023understanding} visualize the loss landscape of FL  while not mode connectivity, and then our work tries to fill this gap of mode connectivity in  FL under data heterogeneity.


\textbf{Mean-field theory.}
A  mean-field theory for the NN dynamics analysis is proposed by \cite{SongMei2018meanfield,SongMei2019meanfield} based on the gradient-flow approximation.
The theory shows that when the neuron number is sufficiently large, the dynamics of stochastic gradient descent (SGD) to optimize NNs can be well approximated by a Wasserstein gradient flow.
This approximation has been studied both in the two-layer case \cite{SongMei2018meanfield,rotskoff2018neural} and the multi-layer case \cite{nguyen2020rigorous,shevchenko20aLandscape}.
In particular,  the approximation is utilized to analyze the dropout stability of SGD solutions by \cite{shevchenko20aLandscape}, and the dropout stability implies mode connectivity of NNs according to  \cite{Kuditipudi2019Explaining_Landscape}. 
However, these results are solely observed in centralized training.
To the best of our knowledge, this paper is the first to expand upon the theoretical results into FL.


\section{Preliminaries}
 
  \begin{figure*}[t]
    \centering
    \includegraphics[width=\textwidth]{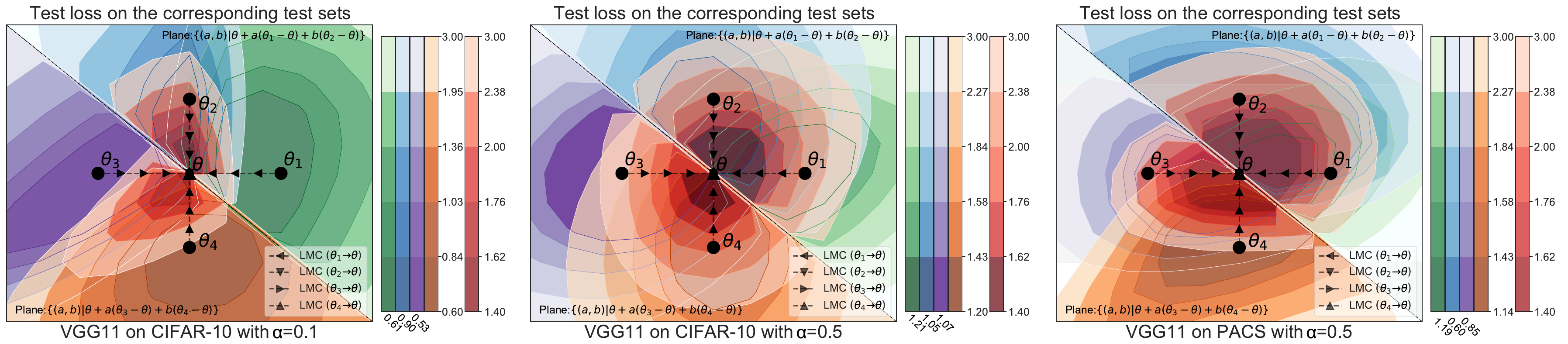}
    \caption{Loss landscapes for the global mode and client modes on their corresponding test sets at the 200-th round (final round), and LMC  from client modes to the global mode. 
      Loss landscape on  VGG11 on CIFAR-10 with  $\alpha=0.1$  ({\textit{left}}), CIFAR-10 with 
 $\alpha=0.5$ ({\textit{middle}}) and on PACS with $\alpha=0.5$ ({\textit{right}}).}
    \label{fig:VGG11_connectivity_cifar10}
\end{figure*}

 \textbf{Mode connectivity.}
The forward function of an NN parameterized by $\bm{\theta} \in \Theta$ is represented as  $f_{\bm{\theta}}: \mathcal{X} \rightarrow \mathcal{Y}$, where $\Theta$ denotes the parameter space of an NN architecture, and $\mathcal{X} \in \mathbb{R}^d$ and $\mathcal{Y} \in \mathbb{R}$ are input and output spaces, respectively. 
When models ${\bm{\theta}_1}$ and ${\bm{\theta}_2}$ have the same neurons but different neuron permutations, they belong to the same mode. 
According to \cite{Kuditipudi2019Explaining_Landscape,shevchenko20aLandscape},  given an NN architecture $\Theta$, a dataset $\mathcal{D}$ and a loss function $\mathcal{L}(\cdot)$, mode connectivity is defined as follows:
 \begin{definition} \label{def:model_connectivity} 
     (Mode connectivity). Two mode parameters $\bm{\theta}_1 \in \Theta$ and $\bm{\theta}_2 \in \Theta$  are $\varepsilon_C$-connected  if there exists a continuous path in parameter space $\pi_{\bm{\theta}}(\nu):[0,1]\rightarrow \Theta $, such that $\pi_{\bm{\theta}}(0)=\bm{\theta}_1$ and $\pi_{\bm{\theta}}(1)=\bm{\theta}_2$ with $\mathcal{L}(\pi_{\bm{\theta}}(\nu))\leq \max(\mathcal{L}(\bm{\theta}_1),\mathcal{L}(\bm{\theta}_2)) + \varepsilon_C$, where $\varepsilon_C$ is the connectivity error.
 \end{definition}
This definition measures the connectivity of two modes by comparing the loss change along a path to its endpoints.
 Moreover,  we consider a path-finding task proposed by \cite{Garipov2018Loss_Surfaces}, which minimizes the expected loss over a uniform distribution on the path as:
 \begin{equation}\label{curve_finding_loss}
  \min   \mathcal{L}_{\rm curve}(\bm{\theta}) = \mathbb{E}_{\nu\sim U(0,1)}\left[\mathcal{L}(\pi_{\bm{\theta}}(\nu))\right],
 \end{equation}
where the path $\pi_{\bm{\theta}}(\nu)$ can be characterized by a Polygonal chain (PolyChain) with one or more bends \cite{Garipov2018Loss_Surfaces}. 
We consider a  PolyChain  case   with one bend as:
\begin{equation}\label{equation:chain_one_bend}
    \pi_{\bm{\theta}}(\nu) =
    \left\{
             \begin{array}{lr}
             2\left[(0.5-\nu)\bm{\theta}_1+ \nu \bm{\theta}_{\rm b}\right], & 0\leq\nu\leq0.5, \\
             2\left[ (\nu -0.5)\bm{\theta}_{\rm b}+(1-\nu)\bm{\theta}_2 \right], &  0.5\leq\nu\leq1,
             \end{array}
\right.
\end{equation}
where $\bm{\theta}_{\rm b}$ is the model of the PolyChain-bend point and can be found by $\min_{\bm{\theta}_{\rm b}}   \mathcal{L}_{\rm curve}(\bm{\theta}_{\rm b})$ in (\ref{curve_finding_loss}). 
We refer to mode connectivity as linear (or non-linear) mode connectivity (LMC) when using linear interpolation (or PolyChain).
Compared with the non-linear counterpart,  LMC is a stronger constraint since LMC requires the two modes to stay in the same basin.

\textbf{Federated leaning (FL).} 
We consider an FL framework with $K$ clients,   each with its own dataset $\mathcal{D}_k   = \{ ({\bm x}_{i_k}, y_{i_k})  \}$, where data samples $ ({\bm x}_{i_k}, y_{i_k}) \sim \mathbb{P}_k:\mathbb{R}^d \times  \mathbb{R} $    are indexed by $i_k \in [n_k]$.
The global dataset is the union of all client datasets and denoted by $\mathcal{D} = \cup_{k=1}^K \mathcal{D}_k \sim \mathbb{P}$ on $ \mathbb{R}^d \times  \mathbb{R}$, which includes  $n = \sum_{k=1}^K n_k$ data samples.
The FL  objective is to minimize the expected global loss $  \mathcal{L}(\bm{\theta}) $   on  $\mathcal{D}$, and is formulated as: 
 \begin{equation} \label{objective:FL}
 \begin{aligned}
 \min_{\bm{\theta}} \mathcal{L}(\bm{\theta}) & =  \mathbb{E}_{({\bm x}, y) \in \mathcal{D}}[  l (\bm{\theta};({\bm x}, y))] \\
  &  =  \sum_{k=1}^K  \frac{n_k}{n} \mathbb{E}_{({\bm x}, y) \in \mathcal{D}_k}[  l_k (\bm{\theta};({\bm x}, y))], 
 \end{aligned}
 \end{equation}
  where $l(\cdot)$ and $l_k(\cdot)$ denote  the global and client loss functions,  respectively.
 The $k$-th client objective is  $\min_{\bm{\theta}_k} \mathcal{L}_k(\bm{\theta}_k)=  \mathbb{E}_{({\bm x}, y) \in \mathcal{D}_k}[  l_k (\bm{\theta}_k;({\bm x}, y))]$.
When $\mathbb{P}_k\neq \mathbb{P}$, FedAvg   \cite{mcmahan17FL}  takes client models $\{\boldsymbol{\theta}_k\}_{k=1}^K$ to solve (\ref{objective:FL}) by minimizing client objectives locally and  obtaining the global model $\boldsymbol{\theta} =\sum_{k=1}^{K} \frac{n_k}{n}\boldsymbol{\theta}_k$  round by round.

\textbf{Experimental setups.}
We explore mode connectivity in FL on classification tasks, including MNIST \cite{lecun1998gradient}, CIFAR-10 \cite{krizhevsky2009learning},  and PACS \cite{li2017deeper}.
The experimental setups used throughout this paper are as follows, unless stated otherwise.
For NN architectures, we consider a two-layer NN, a shallow convolutional neural network (CNN) with two convolutional layers and a deep CNN, i.e., VGG architecture \cite{simonyan2014very}.
For the FL setup,    we consider ten clients participating in FL with a total of 200 rounds.
Client optimizers are SGD with a learning rate of 0.01 and momentum of 0.9, mini-batches of size   50, the local epoch of 1, and dropout of 0.5 for training VGG.
For data heterogeneity,  we follow \cite{kairouz2021advances} and examine label distribution skew by using Dirichlet distribution $Dir(\alpha)$ \cite{yurochkin2019bayesian} to create client label datasets.
A larger $\alpha$ suggests more homogenous distributions across clients. 
On this basis, we also make use of the built-in domain shift of PACS to introduce feature distribution skew across clients.

\section{Mode Connectivity of Client Modes}

In this section, we explore mode connectivity among client modes trained on client datasets, where $\mathbb{P}_k \neq \mathbb{P}_{k^\prime}$ when $k \neq k^\prime$.
We first investigate the geometric landscape of client modes on their corresponding test sets.
Specifically,  we create a hyperplane by using the global mode $\bm{\theta}$ and client modes $\bm{\theta}_k$, $\bm{\theta}_{k^\prime}$, which is represented as $\{(a,b)| \bm{\theta} + a(\bm{\theta}_k -\bm{\theta} ) + b(\bm{\theta}_{k^\prime} - \bm{\theta} )\}$. 
We then use the test sets of clients $k$ and $k^\prime$ to evaluate the performance of all points on the hyperplane.
      Figure \ref{fig:VGG11_connectivity_cifar10}  displays  the loss landscapes obtained on different test sets and distinguishes them by different colors.
We also depict the loss landscape of the global mode (i.e., the initialization of client modes) on the global test set in Figure \ref{fig:VGG11_connectivity_cifar10}.
For comparison, we set the legend scale of the client-mode landscape to be the same in each sub-figure of Figure \ref{fig:VGG11_connectivity_cifar10} except the lowest value, and keep the scale of the global-mode landscape consistent among all the sub-figures.

 \begin{figure}[t]
	\centering  
	\subfigbottomskip=2pt 
	\subfigcapskip=0pt 
 	\subfigure{
		\includegraphics[width=0.46\linewidth,height=90pt]{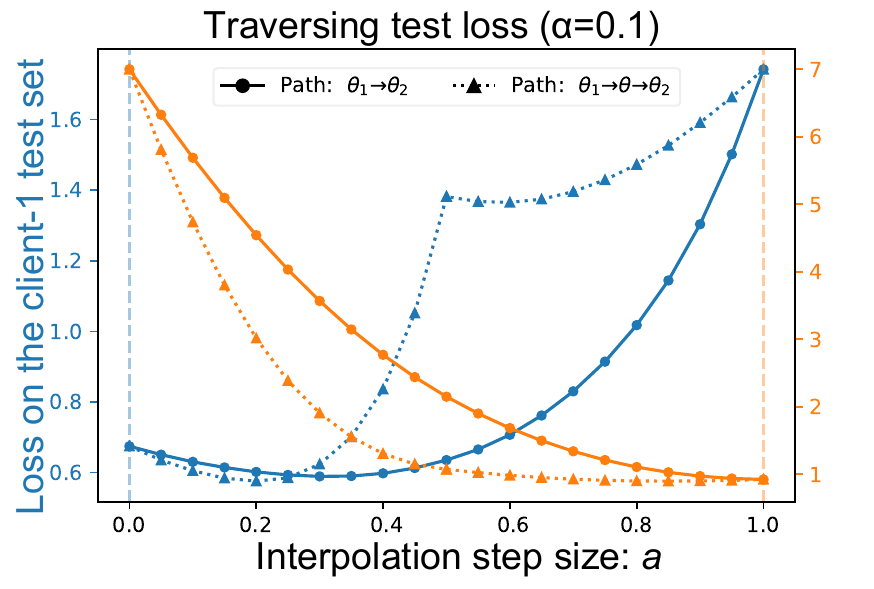}
  \label{Traversing_test_loss_VGG11_E1_a}}
	\subfigure{
		\includegraphics[width=0.46\linewidth,height=90pt]{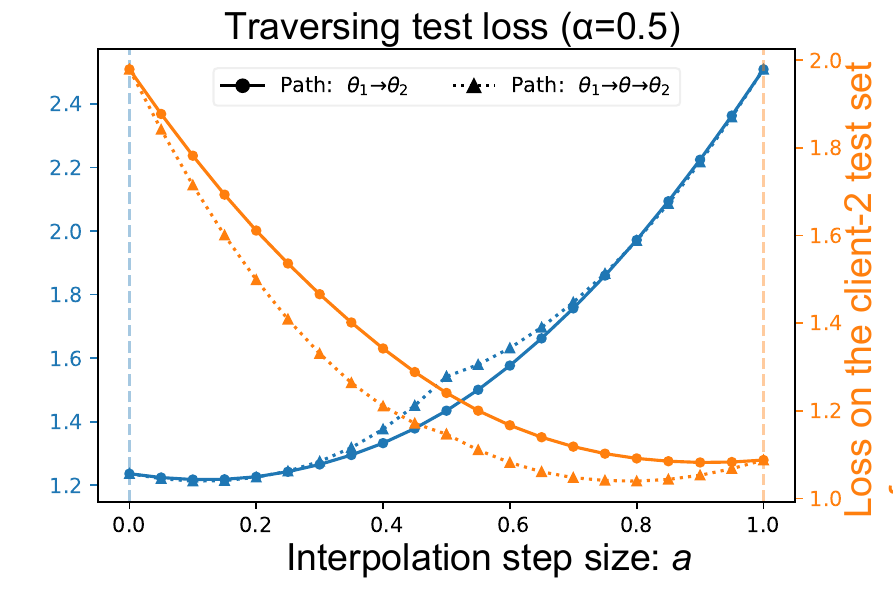}
    \label{Traversing_test_loss_VGG11_E1_b}}
    \caption{Traversing loss  under  $\alpha=0.1$ (\textit{left}) and $\alpha=0.5$ (\textit{right})  from client 1 to client 2.}
  \label{fig:Traversing_test_loss_VGG11_E1}
\end{figure}

\begin{figure}[t]
	\centering  
	\subfigbottomskip=2pt 
	\subfigcapskip=0pt 
 	\subfigure{
		\includegraphics[width=0.45\linewidth, height=80pt]{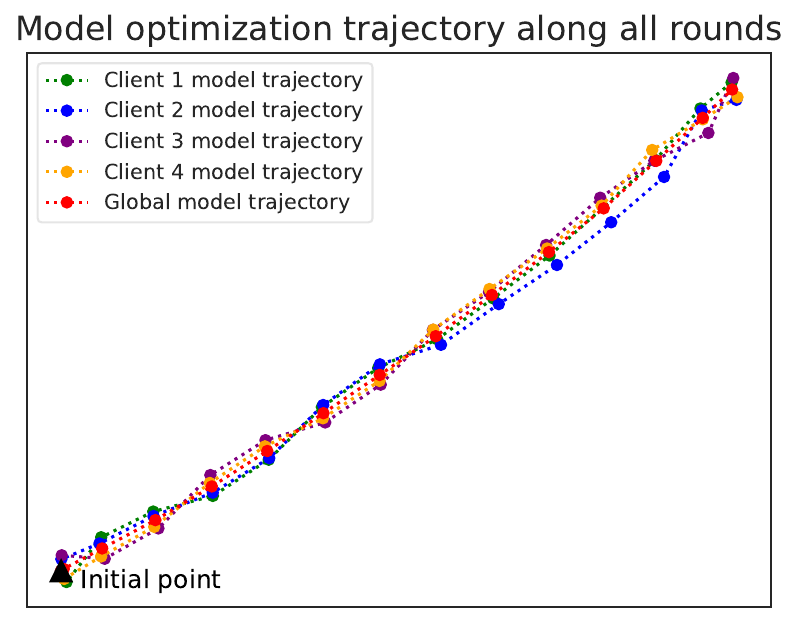}
  \label{optimization_trajectory_VGG11_E1_a}}
	\subfigure{
		\includegraphics[width=0.45\linewidth, height=80pt]{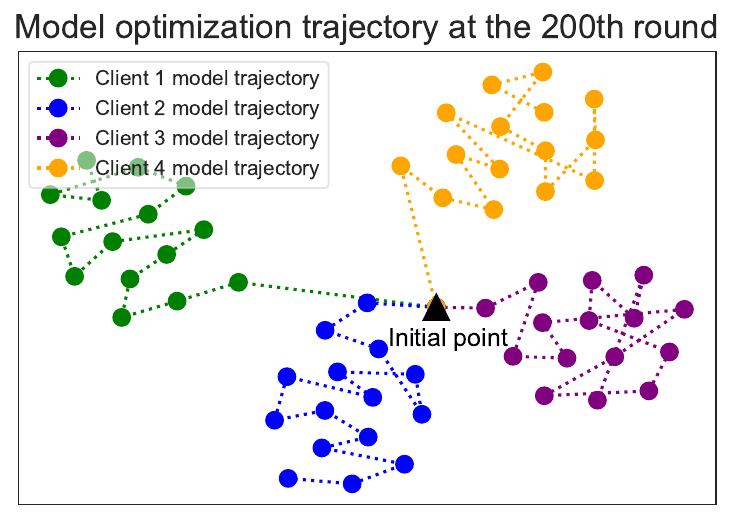}
    \label{optimization_trajectory_VGG11_E1_b}}
    \caption{Client models' optimization trajectory (VGG11) along the training rounds (\textit{left}) and along the local iterations within one round (\textit{right}).}
  \label{fig:optimization_trajectory_VGG11_E1}
\end{figure}

\begin{figure*}[t]
	\centering  
	\subfigbottomskip=0pt 
	\subfigcapskip=0pt 
 	\subfigure{
		\includegraphics[width=0.46\linewidth,height=90pt]{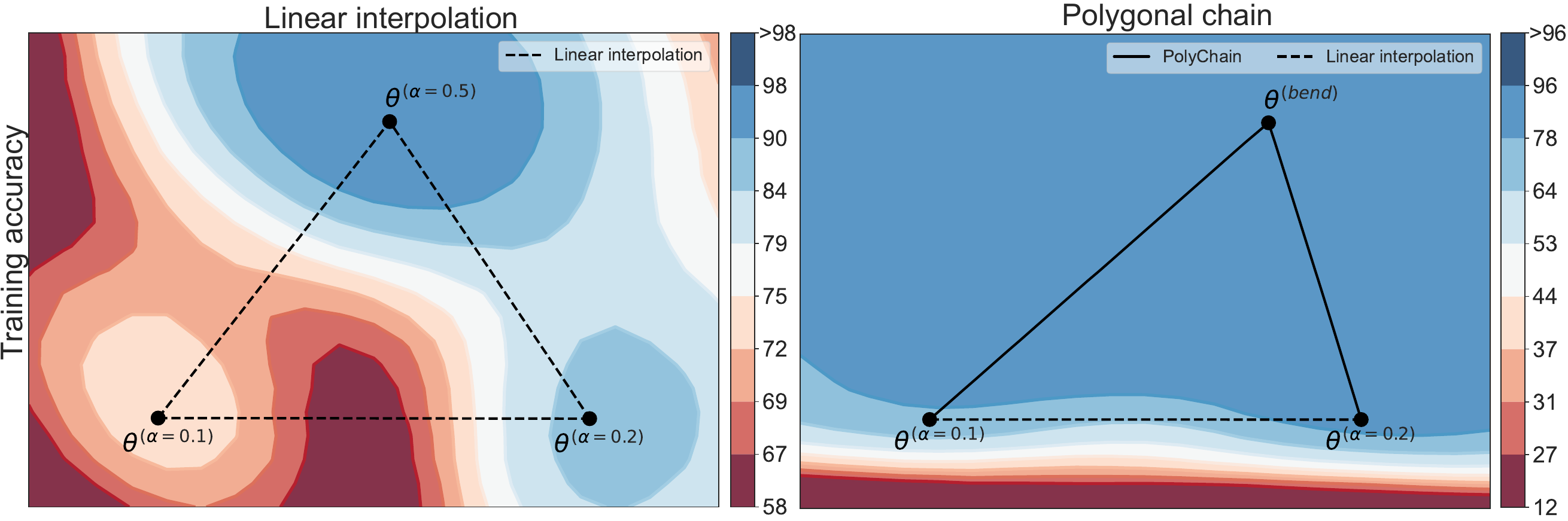}
  \label{fig:global_model_acc_mode_connecitvity_a}}
	\subfigure{
		\includegraphics[width=0.46\linewidth,height=90pt]{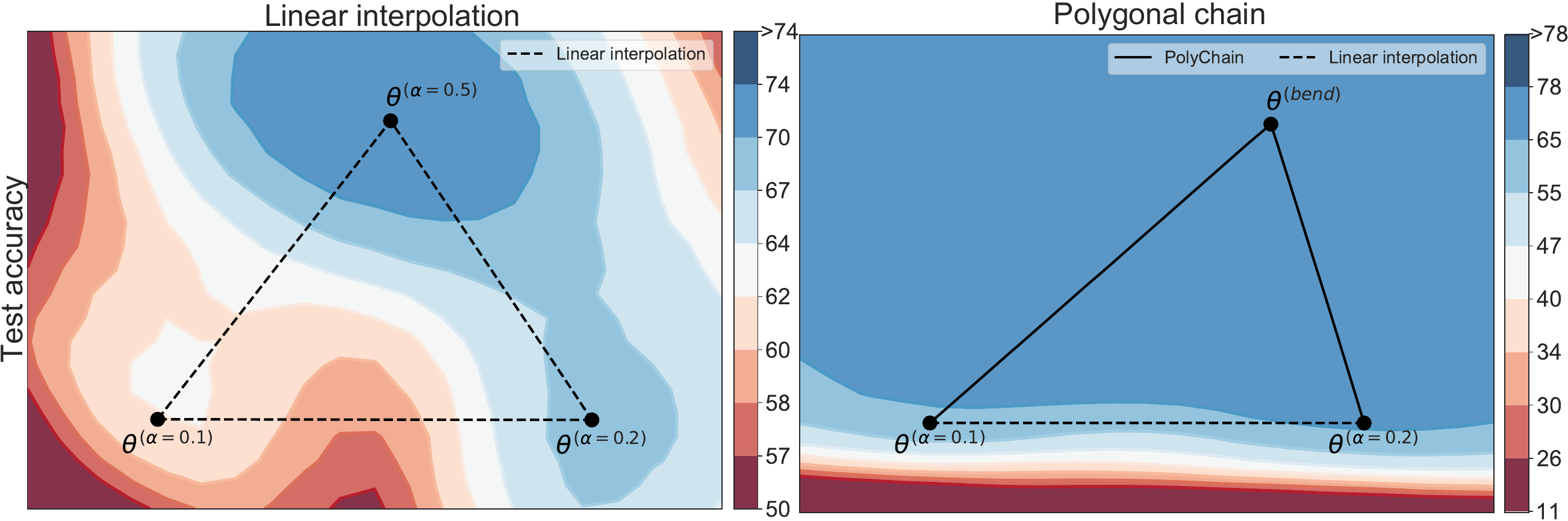}
    \label{fig:global_model_acc_mode_connecitvity_b}}
    \caption{Training accuracy (\textit{{Left}}) and test accuracy  (\textit{{Right}})  with linear connectivity and Polygonal connectivity on global modes.}
  \label{fig:global_model_acc_mode_connecitvity}
\end{figure*}

\textbf{Visualizing mode connectivity on client modes.} 
As shown in Figure \ref{fig:VGG11_connectivity_cifar10}, the comparison between the cases of $\alpha=0.1$ and $\alpha=0.5$ indicates that \textit{when data heterogeneity increases,   low-loss overlaps between client and global modes decrease, and the  number of overlapping solutions to both client and FL objectives reduces}.
Similar results are also observed by comparing the cases of CIFAR-10 and PACS.
Meanwhile, although client modes achieve a lower test loss than the global mode, they stray from the low-loss landscapes of the global mode.
This indicates that client modes are prone to overfitting their own objectives when FedAvg solves (\ref{objective:FL}) even if the global mode provides adequate initialization for them in the final round.
On the other hand, we can linearly connect client models to the global model in all cases, as shown in the LMC of Figure \ref{fig:VGG11_connectivity_cifar10}.
As the loss changes along these linear paths, it may eventually reach the low-loss landscapes of the global mode while still remaining in the low-loss landscapes of client modes.
The above observations reveal that there exist some solutions that meet both client and FL objectives, but they may be overlooked by FedAvg.

\textbf{Traversing loss along the paths connecting client modes.} 
 To delve deeper into the loss landscapes of Figure \ref{fig:VGG11_connectivity_cifar10},  we consider two paths to connect two client modes.
 One path is a linear connection between them (i.e., $\bm{\theta}_k \rightarrow \bm{\theta}_k^{\prime}$), while the other is a PolyChain with the global mode as its bend (i.e., $\bm{\theta}_k \rightarrow \bm{\theta}  \rightarrow \bm{\theta}_k^{\prime}$) as per (\ref{equation:chain_one_bend}).
We follow the paths to traverse the landscape from $\bm{\theta}_k$  to $\bm{\theta}_k^{\prime}$ on their test sets, and show the traversing losses along the paths in Figure \ref{fig:Traversing_test_loss_VGG11_E1}.
For the PolyChain in the case of  $\alpha=0.1$, the intersection point of the two curves deviates from the global mode (i.e., $a=0.5$), where the intersection represents a solution that works effectively on both client datasets.
Moreover,  the traversing-loss gap between the two paths in the case of $\alpha=0.1$ is greater than that of  $\alpha=0.5$.
This means that \textit{when heterogeneous data exist at clients, it becomes more difficult to find solutions that align with both the client and FL objectives}.

\textbf{Client-mode trajectory.}
The trajectory of client modes during optimization in the case of $\alpha=0.1$ is depicted in Figure \ref{fig:optimization_trajectory_VGG11_E1}.
This trajectory is based on t-SNE \cite{van2008visualizing} and covers both training rounds and local iterations within each round.
The figure illustrates that throughout the training process, the distance between client modes is controlled even under serious data heterogeneity, such that they closely surround the global mode.
During the same round, data heterogeneity pushes client modes with the same initialization to be optimized locally along different directions.
Therefore, \textit{the solutions found by client objectives vary when facing data heterogeneity even though they are close to each other in the parameter space}.
 Figure \ref{fig:optimization_trajectory_CNN_E5} shows that the shallow CNN has similar results.


\section{Mode Connectivity of Global Modes}
In this section, we explore the mode connectivity of global modes obtained by different data heterogeneity of the same training task.
We consider two paths, linear interpolation and  PolyChain, to investigate the mode connectivity among global modes on CIFAR-10 under $\alpha=0.1, 0.2$ and $0.5$, as shown in Figure \ref{fig:global_model_acc_mode_connecitvity}.
Note that the initialization of global modes remains the same across these three cases.

\textbf{Visualizing mode connectivity on global modes.} 
In Figure \ref{fig:global_model_acc_mode_connecitvity}, the linear-interpolation landscape reveals that global modes obtained from varying data heterogeneity are situated in distinct basins with an error barrier separating them.
The landscape also shows that when sharing the same initialization, the global-mode basins are within a common region surrounded by high errors.
With more heterogeneous data,  like those involving  $\alpha=0.1, 0.2$, both training and test error barriers become higher, compared with that of $\alpha=0.1, 0.5$.
Furthermore,  we also use a PolyChain found by (\ref{curve_finding_loss}) to connect the global modes of $\alpha=0.1, 0.2$.
Along the PolyChain,  all the solutions keep no barrier in both training and test landscapes.
The above results indicate that \textit{global modes can be connected without any barrier by some paths and they are situated within a manifold}, as illustrated in Figure \ref{fig:toy_example_mode_space}.

\begin{figure}[t]
	\centering  
	\subfigbottomskip=2pt 
	\subfigcapskip=0pt 
 	\subfigure{
		\includegraphics[width=0.46\linewidth,height=90pt]{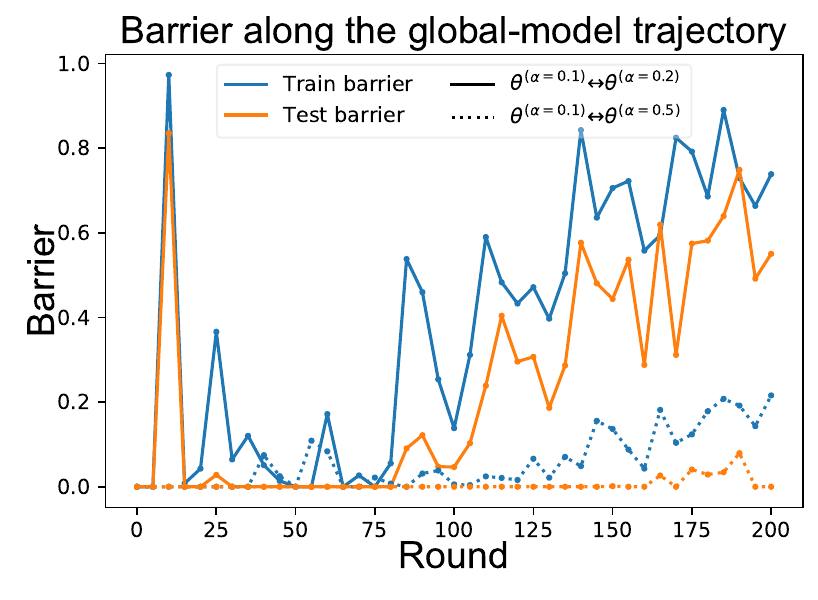}
  \label{fig:barrier_a}}
	\subfigure{
		\includegraphics[width=0.46\linewidth,height=90pt]{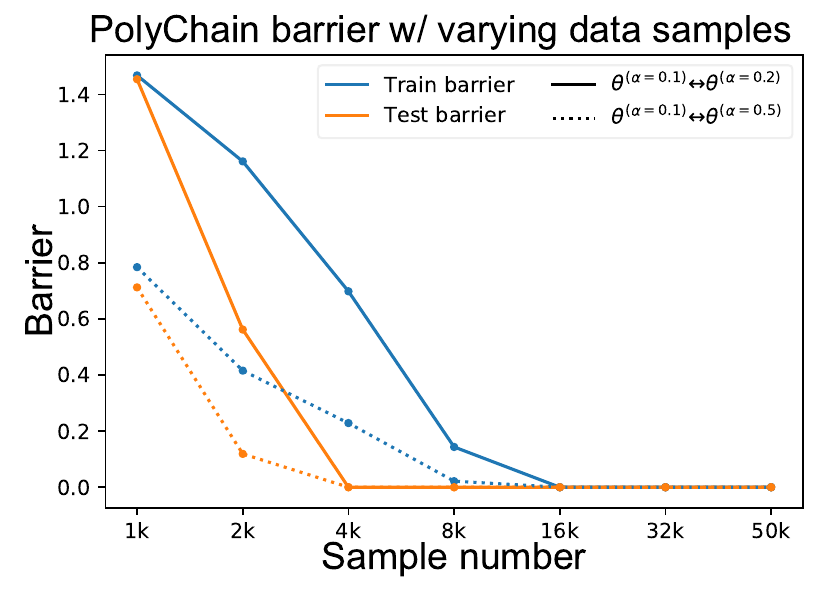}
    \label{fig:barrier_b}}
    \caption{Global-mode barrier with linear interpolation (\textit{left}) and PolyChain (\textit{right}).}
  \label{fig:barrier}
\end{figure}

\textbf{Barrier among global modes.}  
To further quantify the mode connectivity of two global modes ${\bm {\theta}}$ and ${\bm {\theta}^\prime}$, we  define a barrier metric as:
\begin{equation}\label{equation:barrier}
    B({\bm {\theta}},{\bm {\theta}^\prime}) = 
    \max_{a\in[0,1]} \frac{ \left|\mathcal{L} \left(\pi_{\bm{\theta}}(a) \right) -\min(\mathcal{L} ({\bm { \theta}}), \mathcal{L} ({\bm {\theta}}^\prime)) \right|}{ |\mathcal{L} ({\bm { \theta}}) - \mathcal{L} ({\bm {\theta}}^\prime)|} - 1
\end{equation}
where $\pi_{\bm{\theta}}(a)$ is equal to $(1-a) {\bm {\theta}} + a  {\bm{\theta} }^\prime$ and   (\ref{equation:chain_one_bend}) when considering linear interpolation and PolyChain, respectively.
Here,   $B({\bm {\theta}},{\bm {\theta}^\prime})$ measures the largest performance gap along $\pi_{\bm{\theta}}(a)$ comparing with its endpoints ${\bm {\theta}}$ and ${\bm {\theta}^\prime}$.
We say that ${\bm {\theta}}$ and ${\bm {\theta}^\prime}$ are well-connected if the barrier $B({\bm {\theta}},{\bm {\theta}^\prime})$ is close to $0$.

We then use $B({\bm {\theta}},{\bm {\theta}^\prime})$ to explore mode connectivity during the whole training round. 
In Figure \ref{fig:barrier}, the sub-figure on the left shows that more severe data heterogeneity causes greater barriers by comparing the cases of  $\alpha=0.1, 0.2$ and $\alpha=0.1, 0.5$, which persist throughout training.
Meanwhile, as the rounds progress, the barrier shows an increasing trend after an initial oscillation. 
This causes various global modes to reach their own basins, as found in Figure \ref{fig:global_model_acc_mode_connecitvity}.
We also examine the difficulty of finding the global-mode PolyChain when facing varying data heterogeneity.  
Here, the difficulty is quantified by the relationship between barrier $B({\bm {\theta}},{\bm {\theta}^\prime})$ and data sample to optimize  (\ref{curve_finding_loss}) and is depicted in the right sub-figure of Figure \ref{fig:barrier}.
The sub-figure illustrates that more samples are needed to find the specific PolyChain as data heterogeneity increases, resulting in more incredible difficulty in connecting global modes with no error barrier.

\textbf{Function dissimilarity and distance of global modes.}
According to \cite{entezari2022the}, permutation invariance may contribute to a  linear-interpolation barrier between NN solutions in centralized training.
When considering FL,   data heterogeneity may disturb the permutation of FL solutions, leading to the barriers found in Figure \ref{fig:global_model_acc_mode_connecitvity}.
To verify whether the FL solutions obtained from different data heterogeneity belong to the same global mode,   we measure their function dissimilarity on the left-hand side of Figure \ref{fig:weight_distance_function_similarity}.
The figure shows that function dissimilarity remains $>0.2$ among these FL solutions even when the data heterogeneity is weakened.
This means that the FL solutions belong to different global modes.
Furthermore, Figure \ref{fig:weight_distance_function_similarity} presents that the distance between these FL solutions is smaller than the distance between them and their initial points.
This means that different global modes with the same initialization are close to each other.
Combined with the above results, it can be inferred that \textit{these FL solutions belonging to various global modes are adjacent and in a common region, but they cannot be linearly connected}.

\begin{figure}[t]
	\centering  
	\subfigbottomskip=2pt 
	\subfigcapskip=0pt 
 	\subfigure{
		\includegraphics[width=0.46\linewidth,height=90pt]{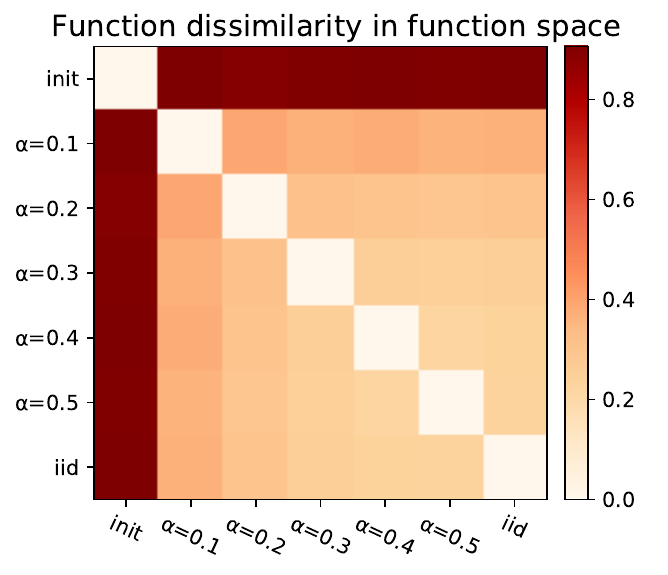}
  \label{fig:weight_distance_function_similarity_a}}
	\subfigure{
		\includegraphics[width=0.46\linewidth,height=90pt]{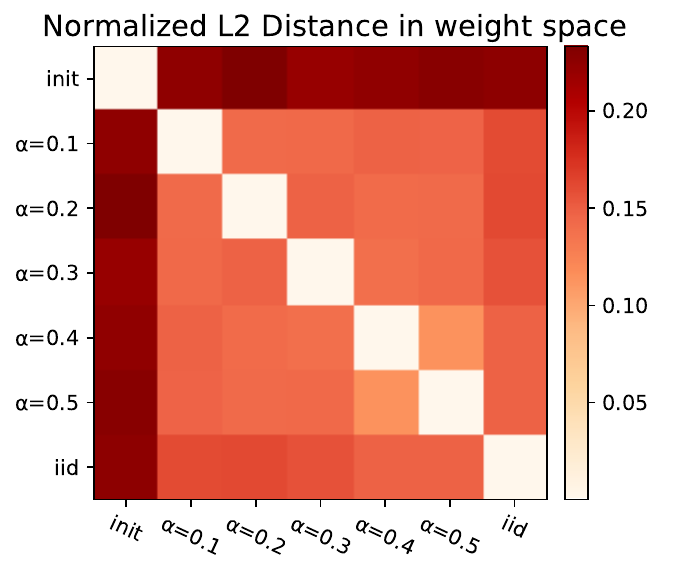}
    \label{fig:weight_distance_function_similarity_b}}
    \caption{Global modes in function space (\textit{left}) and   weight space (\textit{right}). Function dissimilarity is computed by the fraction of labels on which the predictions from different global modes disagree when given the same inputs.  L2 distance is normalized by the norm of model initialization.  }
  \label{fig:weight_distance_function_similarity}
\end{figure}

\section{Mode Connectivity Analysis under Data Heterogeneity}
  As indicated by our experimental results, a  sufficient condition to find the solutions meeting both client and FL objectives (\ref{objective:FL}) is that the distributions represented by client modes $\boldsymbol{\theta}_k$ are the same as the global mode $\boldsymbol{\theta}$.
 However, satisfying this condition is generally challenging since data heterogeneity induces client-gradient drifts \cite{zhao2018federated} and then implicit output bias among client modes  \cite{zhou2023understanding}.
Building upon the drifts and bias,  we consider  $K$ client modes $\{\boldsymbol{\theta}_k\}_{k=1}^K$  and formulate  data heterogeneity as:
 \begin{definition} \label{def:data_heterogeneity}
     (Data heterogeneity). Given  a global mode   ${\boldsymbol{\theta}}=\sum_{k=1}^{K} \frac{n_k}{n}\boldsymbol{\theta}_k$  and data samples $({\bm x}, y) \sim \mathbb{P}$,  for $ \forall k \in [K]$, the global-mode output   $\Bar{y}  = f_{{\boldsymbol{\theta}} }(\boldsymbol{x} ) $ and its gradient $\nabla_{{\boldsymbol{\theta}} }(\boldsymbol{x} )$ differ from the client-mode output $ \hat{y}_{k} = f_{{\boldsymbol{\theta}}_{k}}(\boldsymbol{x})$ and its gradient $\nabla_{{\boldsymbol{\theta}}_{k}}(\boldsymbol{x})$, respectively,  if the client distribution   $\mathbb{P}_k$ differs from the global distribution $\mathbb{P}$.
 \end{definition}
 In this section,  we will utilize dropout stability to analyze mode connectivity of global modes.
First, let us  define it as:
  \begin{definition} \label{def:dropout_stability}
     (Dropout stability).  Given   a model $\boldsymbol{\theta}$ with $N$ neurons and its dropout networks   $\boldsymbol{\theta}_{\mathrm{S}}$ with $\mathcal{A} \subseteq[N]$, the model $\boldsymbol{\theta}$ is $\varepsilon_{\mathrm{D}}$-dropout stable if
$\left|\mathcal{L}_N(\boldsymbol{\theta})-\mathcal{L}_{|\mathcal{A}|}\left(\boldsymbol{\theta}_{\mathrm{S}}\right)\right| \leq \varepsilon_{\mathrm{D}}$.
 \end{definition}
 
Then, we adopt the mean-field theory developed by \cite{SongMei2018meanfield} to analyze dropout stability under data heterogeneity.
Mean-field theory shows that the trajectory of  SGD   can be approximated by a partial differential equation (PDE) called distributional dynamic (DD).
Namely, under the assumption of enough neurons (i.e., $N$ is large enough) and one-pass data processing (i.e., for all $i \in [n] $, samples $({\bm x}_{i}, y_{i})  \in \mathcal{D} \sim \mathbb{P} $ are independent and identically distributed),  the SGD trajectory of  $N$ neurons are close to the movement of $N$ i.i.d particles following the description of DD.
Here,  we consider a two-layer NN with $N$ hidden  neurons, denoted as  ${\bm{\theta}}=({\bm{\theta}_1,\dots,{\bm{\theta}_N}})$, where  ${\bm{\theta}_i} \in \mathbb{R}^D$ for all $i\in[N]$.
The  forward function of ${\bm{\theta}}$ is represented as:
\begin{equation}\label{equation:twolayer_network_function}
    f_{\bm{\theta}}({\bm x})=1/N \sum_{i=1}^N \sigma_*({\bm x};{\bm{\theta}_i}),
\end{equation}
 where    $\sigma_*: \mathbb{R}^d \times\mathbb{R}^D \rightarrow \mathbb{R}$ is the activation function.
For simplicity, we focus on the    ReLU  activation and the  mean-square loss, i.e.,  $\sigma_*({\bm x};{\bm{\theta}_i})=\sigma(\langle{\bm x},{\bm{\theta}_i}\rangle)$ and $l (\bm{\theta};({\bm x}, y))= (y - f_{\bm{\theta}}({\bm x}))^2$, 
which can be extended to other activation functions and losses in \cite{SongMei2018meanfield,SongMei2019meanfield}.

Meanwhile,  we consider  FedAvg  with each client  undergoing $p$ local iterations in the $m$-th round, and represent the global-mode update  on the $i$-th neuron $\boldsymbol{\theta}_i$ as $\Delta \boldsymbol{\theta}_i^{(m)} = \boldsymbol{\theta}_i^{m+1,0}
 - \boldsymbol{\theta}_i^{m,0}$ to get: 
\begin{equation} \label{equation:update_FedAvg_simple}
\begin{aligned}
\Delta \boldsymbol{\theta}_i^{(m)} = & 2 \sum_{\tau=0}^p  {s}^{m,\tau} 
   (y^{m,\tau}-\Bar{y}^{m,\tau})  \nabla_{\Bar{\boldsymbol{\theta}}_i} \sigma(\boldsymbol{x}^{m,\tau} ; \Bar{\boldsymbol{\theta}_{i}}^{m,\tau})  \\
&  +       2 \sum_{\tau=0}^p {s}^{m,\tau} \boldsymbol{n}_{i}^{m,\tau},  
\end{aligned}
\end{equation}
where  $\tau \in [p]$ denotes the $\tau$-th local iteration, ${s}^{m,\tau} $ denotes the step size    at the $\tau$-th iteration,
$\Bar{\boldsymbol{\theta}_{i}}^{m,\tau}=\sum_{k=1}^{K}  \frac{n_k}{n} \boldsymbol{\theta}_{i,k}^{m,\tau}$ denotes the weighted average  of all the $i$-th client neurons, 
$\Bar{y}^{m,\tau} = 1/N \sum_{i=1}^N \sigma({\boldsymbol{x}^{m,\tau}};{\Bar{\boldsymbol{\theta}_{i}}^{m,\tau}})$ denotes the output of the weighted-average model $\Bar{\boldsymbol{\theta}}^{m,\tau}$  based on one-pass data $({\boldsymbol{x}^{m,\tau}},{\boldsymbol{y}^{m,\tau}}) \in \mathbb{P}$, and $ \boldsymbol{n}_{i}^{m,\tau}$ denote the noise induced by data heterogeneity.   
Specifically, we follow Definition \ref{def:data_heterogeneity} and represent $ \boldsymbol{n}_{i}^{m,\tau}$ (the superscript is omitted for brevity) as:
\begin{equation}\label{equation:noise}
\begin{aligned} 
         \boldsymbol{n}_{i} =    \sum_{k=1}^{K} \frac{n_k}{n}  &  \left[(y_k-\hat{y}_k) \right. \underbrace{[  \nabla_{\boldsymbol{\theta}_{i,k}} \sigma(\boldsymbol{x}_k; \boldsymbol{\theta}_{i,k})  -  \nabla_{\Bar{\boldsymbol{\theta}}_i} \sigma(\boldsymbol{x}_k ; \Bar{\boldsymbol{\theta}_{i}})]}_{\text{gradient drift}} 
         \\ &    +         \underbrace{(\Bar{y}_k-\hat{y}_k)}_{\text{output bias}}  \left.\nabla_{\Bar{\boldsymbol{\theta}}_i} \sigma(\boldsymbol{x}_k ; \Bar{\boldsymbol{\theta}_{i}})  \right],
\end{aligned}
\end{equation}
 where  the gradient drift and output bias depend on data heterogeneity, i.e., $(x_k,y_k) \in \mathbb{P}_k $ and $ \in \mathbb{P}$, but $\mathbb{P}_k \neq \mathbb{P}$. 
 
Furthermore, the trajectory of $\Bar{\boldsymbol{\theta}_{i}}$ in (\ref{equation:update_FedAvg_simple}) subsumes the trajectory of the global model ${\boldsymbol{\theta}_{i}}$ since FedAvg performs weighted averaging on the global model at each round.
Therefore, we represent  the global-mode trajectory as follows:
\begin{equation}  \label{equation:update_FedAvg_onestep}
    \boldsymbol{\theta}_i^{\tau+1}  =   \boldsymbol{\theta}_i^{\tau}+
             2 s_\tau\left(y^\tau-\Bar{y}^{\tau}\right) \nabla_{\boldsymbol{\theta}_i} \sigma\left(\boldsymbol{x}_\tau ; \boldsymbol{\theta}_i^\tau\right)    +2 {s}^{\tau} { \boldsymbol{n}}_i^\tau,
\end{equation}
where  $\tau \in [Mp] $ when given a total round $M$ and $\boldsymbol{n}_i^\tau =0$ when $\tau\mod p \in [M]$. See the appendix for detailed proof.
 
We make the following assumptions  as per \cite{SongMei2018meanfield}:
\begin{assumption}\label{assumption1}
The step size $s_\tau$ is denoted as $s_k = \alpha \xi(\tau \varepsilon)$, where $\xi: \mathbb{R}_{\geq 0} \rightarrow \mathbb{R}_{>0}$ is bounded by $C_1$ and $C_1$-Lipschitz.
\end{assumption}

 \begin{assumption}\label{assumption2}
For  $(\boldsymbol{x},y) \sim \mathbb{P}$,   the label $y$ and  the activation function $ \sigma(\boldsymbol{x}, \boldsymbol{\theta})$    with   $C_2$ sub-Gaussian gradient  are   bounded by $K_2$. 
\end{assumption}

 \begin{assumption}\label{assumption3}
   The functions $v(\boldsymbol{\theta})=-\mathbb{E}\{y \sigma(\boldsymbol{x}, \boldsymbol{\theta}) \}$ and $u(\boldsymbol{\theta}_1,\boldsymbol{\theta}_2)=\mathbb{E}\{  \sigma(\boldsymbol{x}, \boldsymbol{\theta}_1)\sigma(\boldsymbol{x}, \boldsymbol{\theta}_2)\}$  are  differentiable, and their    gradients are bounded by $C_3$ and $C_3$-Lipschitz.
\end{assumption}

 \begin{assumption}\label{assumption4}
The initial condition   $\boldsymbol{\theta}_i^0$   is $K^2_4/D$-sub-Gaussian. Let $v(\boldsymbol{\theta})\in C^4(\mathbb{R}^D)$ and   $u(\boldsymbol{\theta}_1,\boldsymbol{\theta}_2)\in C^4(\mathbb{R}^D)$, and  $\nabla u^k_{\boldsymbol{\theta}_1}(\boldsymbol{\theta}_1,\boldsymbol{\theta}_2)$ is   uniformly bounded for $0 \leq k \leq 4$.
\end{assumption}

 \begin{assumption}\label{assumption5}
 For $m \in [M]$ and $\tau \in [p]$, the noise  $\boldsymbol{n}_{i}^{\tau}=\sqrt{   \beta h(\alpha)}\Bar{\boldsymbol{n}}_i^{\tau}$, 
 where $\Bar{\boldsymbol{n}}_i^\tau \sim \mathcal{N}\left(0, \boldsymbol{I}_D\right)$.
\end{assumption}
Assumptions \ref{assumption1}-\ref{assumption4} denote the requirements of mean-field theory on the learning rate $s_k$,  data distribution $(\boldsymbol{x}, y) \sim \mathbb{P}$,  activation function $\sigma$, and initialization $\rho_0$.
Assumption \ref{assumption5} indicates that the noise $\boldsymbol{n}_{i}$ in (\ref{equation:noise}) depends on data heterogeneity. 
Specifically, a generalized function $h(\alpha)$ is taken to measure the impact of data heterogeneity on $\boldsymbol{n}_{i}$, including the gradient drift and output bias, where $ \beta h(\alpha) \in [\beta,\infty)$ and $1/\alpha$ denotes the degree of data heterogeneity across clients (i.e., smaller $\alpha$, larger $h(\alpha))$.

We take the mean-field theory to quantify the difference between the global-mode trajectory loss and the DD-solution loss and have:
\begin{lemma}\label{lemma:PDE_SGD}
(Mean field approximation.)
 Assume that conditions 1-5 hold,      the solution $(\rho_t)_{t\geq0}$ of a PDE  with initialization $\rho_0$    can  approximate the update trajectory of the global mode $(\bm{\theta}_i^\tau)$   as (\ref{equation:update_FedAvg_onestep}) with initialization $\bm{\theta}_i^0 \sim \rho_0$ and unchanged data heterogeneity $h(\alpha)$ throughout   training. 
 When  $\beta=2s^{\tau}\varepsilon/D$,  $  h(\alpha)\leq  C_5$, $T\geq 1$ and $\varepsilon\leq1$,  there exists a constant $C$ (depending solely on the constants $C_i$ of assumptions 1-4) such that
\begin{align*}
  &  \sup_{\tau\in[0,T/\varepsilon]\cap \mathbb{N}}  \left|\mathcal{L}_N(\bm{\theta}^\tau) - \mathcal{L}(\rho_{\tau\varepsilon})\right|   \\
\leq &C e^{C T }\sqrt{h(\alpha)}\left(\frac{\sqrt{\log  N }+z}{\sqrt{N}}+\sqrt{\varepsilon}(\sqrt{D+\log  N }+z)\right),
\end{align*}
with probability at least $1-e^{-z^2}$.
\end{lemma}
 Lemma \ref{lemma:PDE_SGD} shows when the number of neurons $N$ is sufficiently large and the step size $\varepsilon$ is sufficiently small, the global-mode neurons $\boldsymbol{\theta}_i^\tau$ obtained by running $\tau$ steps as (\ref{equation:update_FedAvg_onestep}) can be approximated as $N$ i.i.d. particles that evolve according to the DD at time $\tau\varepsilon$.
Then, we utilize Lemma \ref{lemma:PDE_SGD} to show that the global mode remains dropout-stable even when faced with data heterogeneity.

\begin{theorem}\label{theorem:dropout_stability_FL}
   (Dropout stability under data heterogeneity.) Assume that conditions \ref{assumption1}-\ref{assumption5} hold, and fix $T \geq 1$ and $h(\alpha)\geq C_5$. Let the global model $\boldsymbol{\theta}^\tau$ be obtained by running $\tau$ steps as (\ref{equation:update_FedAvg_onestep}) with data $\left\{\left(\boldsymbol{x}_i, y_i\right)\right\}_{i=0}^\tau {\sim} \mathbb{P}$ and initialization $\rho_0$. Then, the following results hold:
 Pick $\mathcal{A} \subseteq[N]$ independent of $\boldsymbol{\theta}^\tau$. Then, with probability at least $1-e^{-z^2}$, for all $\tau \in[T / \varepsilon], \theta^\tau$ is $\varepsilon_{\mathrm{D}}$-dropout stable with $\varepsilon_{\mathrm{D}}$ equal to
$$
C e^{C T}\sqrt{h(\alpha)}\left(\frac{\sqrt{\log |\mathcal{A}|}+z}{\sqrt{|\mathcal{A}|}}+\sqrt{\varepsilon}(\sqrt{D+\log N}+z)\right),
$$
where the constant $C$ depends only on the constants $C_i$ in the assumptions \ref{assumption1}-\ref{assumption4}.
\end{theorem}

Following \cite{Kuditipudi2019Explaining_Landscape}, we demonstrate that dropout-stable NNs in FL have mode connectivity as follows.
 
\begin{theorem}\label{theorem:mode_connectivity_FL}
 (Mode connectivity under data heterogeneity.)
Under Theorem \ref{theorem:dropout_stability_FL},   fix $T\geq 1, T^{\prime} \geq 1$ and let $\boldsymbol{\theta}$ and $ \boldsymbol{\theta}^{\prime}$ be the global models obtained by running running $\tau$ steps as (\ref{equation:update_FedAvg_onestep}) with data heterogeneity $h(\alpha)$ and $h(\alpha^\prime)$ and initialization $\rho_0$ and $\rho_0^{\prime}$, respectively. 
Then, with probability at least $1-e^{-z^2}$, for all $\tau \in[T / \varepsilon]$ and $\tau^{\prime} \in\left[T^{\prime} / \varepsilon\right], \boldsymbol{\theta}^\tau$ and ${\boldsymbol{\theta}^\prime}^{\tau^{\prime}}$ are $\varepsilon_{\mathrm{C}}$-connected with $\varepsilon_{\mathrm{C}}$ equal to
\begin{equation*}
    \begin{aligned}
    C e^{C T_{\max }}\max(\sqrt{h(\alpha)},\sqrt{h(\alpha^\prime)})\left(\frac{\sqrt{\log N}+z}{\sqrt{N}} \right. \\
  \quad\quad\quad \left.  +\sqrt{\varepsilon}(\sqrt{D+\log N}+z)\right),
\end{aligned}
\end{equation*}
where $T_{\max }=\max \left(T, T^{\prime}\right)$. Furthermore, the path connecting $\boldsymbol{\theta}^k$ with ${\boldsymbol{\theta}^{\prime}}^{\tau^{\prime}}$ consists of 7 line segments.  
\end{theorem}

Theorem \ref{theorem:mode_connectivity_FL} shows that the global modes obtained from FL under different data heterogeneity can be connected, as shown in Figure \ref{fig:global_model_acc_mode_connecitvity}, and then mode-connectivity error depends on the data heterogeneity $h(\alpha)$ and neuron numbers $N$.
See the appendix for the detailed proof of Theorems \ref{theorem:dropout_stability_FL} and \ref{theorem:mode_connectivity_FL}.

 \section{Numerical Results}
To verify our analysis results, we undertake two classification tasks: \textit{i)} using a two-layer NN on MNIST; \textit{ii)} using a VGG11 network from scratch on CIFAR-10. 
The NNs take ReLU activations for both tasks and are optimized based on the cross-entropy loss.
For the FL setup,  we consider ten clients  with $50/200$ training rounds for MNIST/CIFAR-10.
Client optimizers are SGD  with a learning rate of 0.02/0.01  for MNIST/CIFAR-10, where the mini-batch size is 50 and the number of local iterations is 10.
 Meanwhile, within the same task,  the NN initialization is consistent under varying data heterogeneity and independent of neuron number $N$, which aligns with the theoretical assumptions.
 We perform 20 independent trials to measure mean value and standard deviation in Figures \ref{fig:dropout_comparision} and \ref{fig:width_comparision}.

\textbf{Dropout stability under varying data heterogeneity.}
\begin{figure}[t]
	\centering  
	\subfigbottomskip=2pt 
	\subfigcapskip=0pt 
 	\subfigure{
		\includegraphics[width=0.46\linewidth,height=90pt]{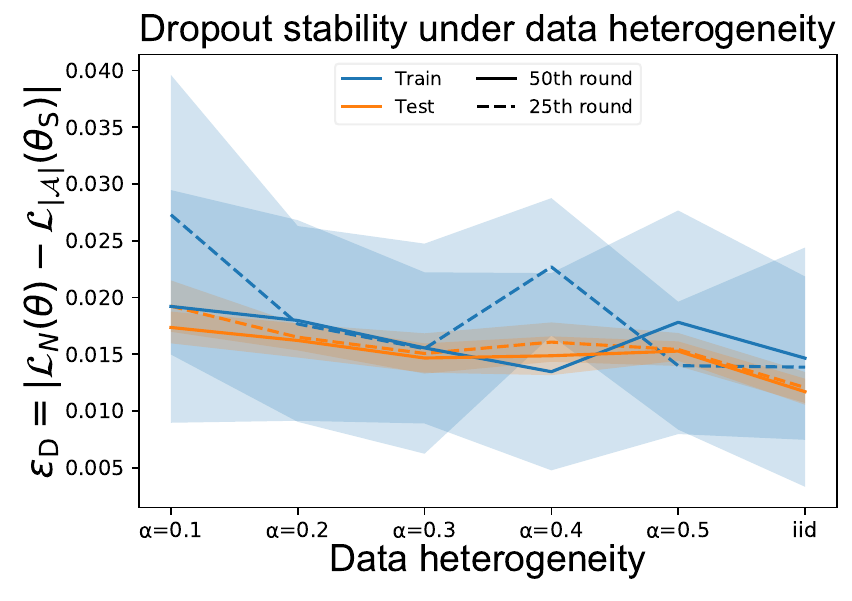}
  \label{fig:dropout_comparision_a}}
	\subfigure{
		\includegraphics[width=0.46\linewidth,height=90pt]{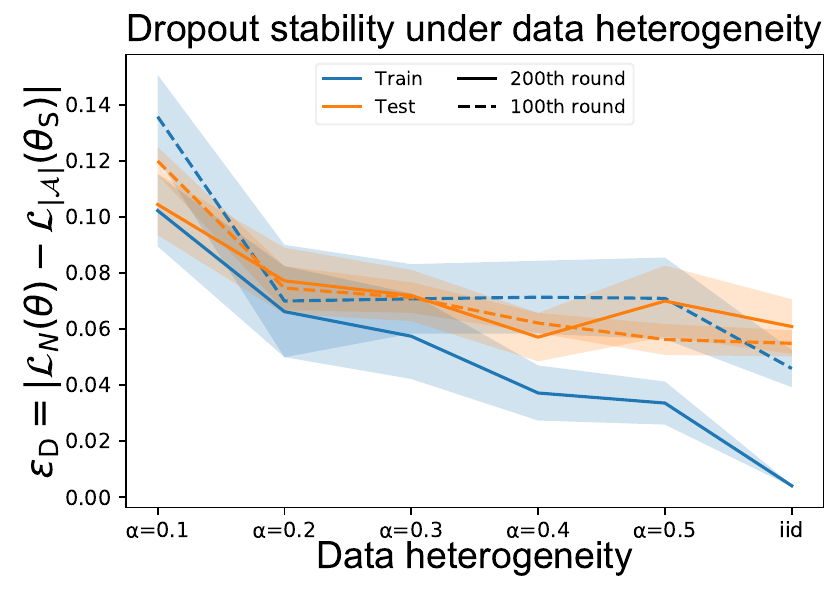}
    \label{fig:dropout_comparision_b}}
    \caption{Training/test dropout stability under varying data heterogeneity on MNIST with a two-layer NN (\textit{left}) and CIFAR-10 with VGG11 (\textit{right}).}
  \label{fig:dropout_comparision}
\end{figure}

Figure \ref{fig:dropout_comparision} compares the dropout error $ \varepsilon_{\mathrm{D}}$ under varying data heterogeneity as per Definition \ref{def:dropout_stability} on the training dataset (blue curve) and test dataset  (orange curve).
For MNIST, the neuron number of the two-layer NN is $N=1600$, and for CIFAR-10, VGG11 follows the standard setup in \cite{simonyan2014very}.
As expected from Theorem \ref{theorem:dropout_stability_FL}, the dropout error shows a downward trend as client datasets become more homogeneous.
For example,  the case of $\alpha=0.1$   has a more significant training/test dropout error than the i.i.d case at the different training stages of both tasks.
When the same heterogeneity setting (e.g., $\alpha=0.1$) is adopted, the effect of data heterogeneity on CIFAR-10 is stronger than that on MNIST.
This leads to a larger dropout error in CIFAR-10, where the dropout error is $\varepsilon_{\mathrm{D}} < 0.04$ for MNIST and $\varepsilon_{\mathrm{D}} < 0.15$ for CIFAR-10.
Moreover,  we also consider two training stages in Figure \ref{fig:dropout_comparision}, including the middle stage (dashed curve) and the ending stage (solid curve).
For both tasks, the dropout-error gap between the two stages keeps small, i.e., $\varepsilon_{\mathrm{D}} < 0.01$ for MNIST and $\varepsilon_{\mathrm{D}} < 0.05$ for CIFAR-10.
As expected, this indicates that the mean-field theory can view the dropout-out dynamics as long as the training time is sufficient.

 \textbf{Mode connectivity under varying neuron numbers and data heterogeneity. }
\begin{figure}[t]
	\centering  
	\subfigbottomskip=2pt 
 	\subfigure{
		\includegraphics[width=0.46\linewidth,height=90pt]{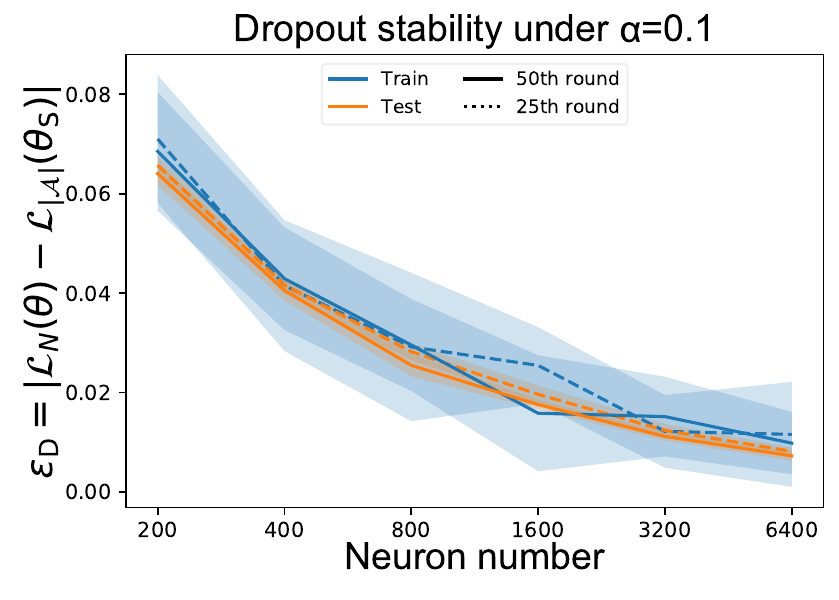}
  \label{fig:width_comparision_a}}
	\subfigure{
		\includegraphics[width=0.46\linewidth,height=90pt]{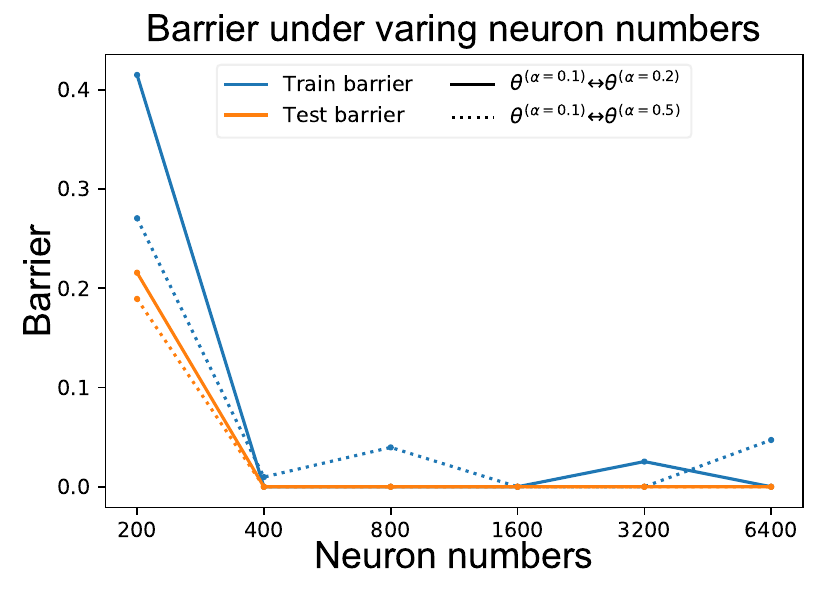}
    \label{fig:width_comparision_b}}
    \caption{Dropout stability under varying neuron numbers (\textit{left}) and linear-interpolation barrier under varying neuron numbers and data heterogeneity   (\textit{right}).}
  \label{fig:width_comparision}
\end{figure}
Figure \ref{fig:width_comparision} takes dropout stability (left sub-figure) and linear-interpolation barrier (right sub-figure) to validate our analysis on mode connectivity of global modes in FL.
 Theorem \ref{theorem:mode_connectivity_FL} shows that dropout stability implies mode connectivity, and the connectivity error $\varepsilon_{\mathrm{C}}$ depends on the neuron number and data heterogeneity.
For the MNIST under $\alpha=0.1$, we consider a two-layer NN and plot the relationship between its dropout error $\varepsilon_{\mathrm{D}}$ and neuron number $N$.  
As expected, the training/test dropout error rapidly decreases as the width of the network grows. 
The dropout error for $N = 6400$ is less than $0.01$, while for $N = 200$, it is up to $0.06$.
Then, we take the linear-interpolation barrier to explore the connectivity error $\varepsilon_{\mathrm{C}}$, and connect the global mode obtained in the case $\alpha=0.1$ to that of the cases  $\alpha=0.2$ (solid curve) and $\alpha=0.5$ (dashed curve).
With the expansion of the network, the barrier (i.e., connectivity error $\varepsilon_{\mathrm{C}}$)  diminishes rapidly, which is consistent with dropout error $\varepsilon_{\mathrm{D}}$ v.s. neuron number $N$.
Furthermore, the barrier of $\alpha=0.1$ to $\alpha=0.2$ is higher than that of $\alpha=0.1$ to $\alpha=0.5$ when $N=200$.
This is expected because when $N$ is small, the connectivity error $\varepsilon_{\mathrm{C}}$ is amplified by the effect of data heterogeneity $h(\alpha)$, i.e., the effect of data heterogeneity amortized to neurons would be even greater.

 \textbf{Neuron noise $\boldsymbol{n}_i$ under varying data heterogeneity.}
 \begin{figure}[t]
	\centering  
	\subfigbottomskip=2pt 
 	\subfigure{
		\includegraphics[width=0.46\linewidth,height=90pt]{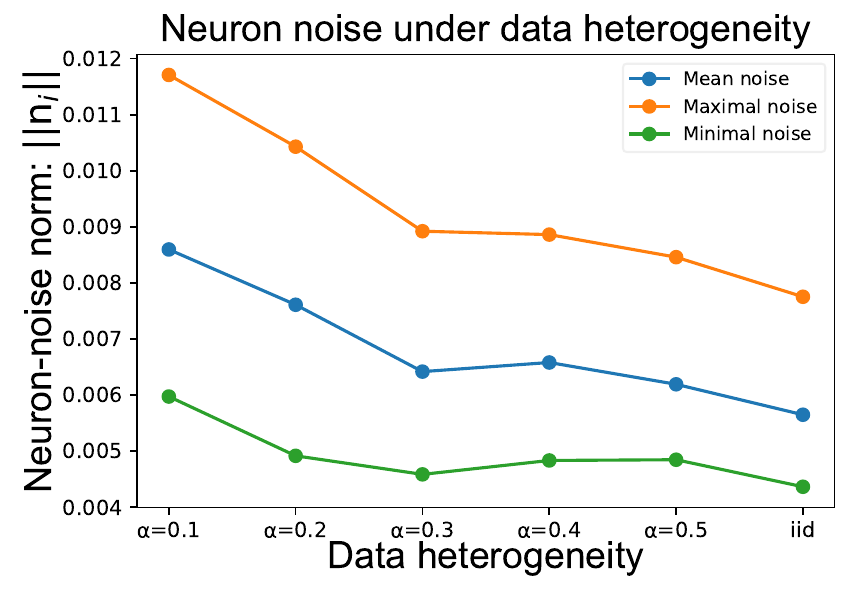}
  \label{fig:noise_comparision_a}}
	\subfigure{
		\includegraphics[width=0.46\linewidth,height=91pt]{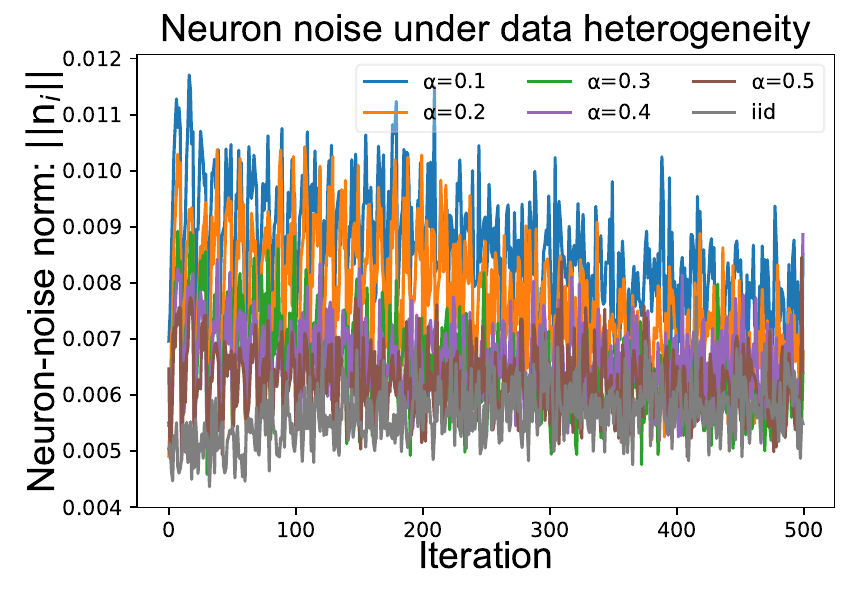}
    \label{fig:noise_comparision_b}}
    \caption{Neuron noise statistics under data heterogeneity (\textit{left}) and neuron noise along the whole training (\textit{right}).}
  \label{fig:noise_comparision}
\end{figure}
Figure \ref{fig:noise_comparision} illustrates the L2 norm of the global-mode update noise under varying data heterogeneity in (\ref{equation:noise}) along all iterations in the right sub-figure, where the mean, maximal, and minimal noise norms are reported in the left sub-figure.
As shown in the left sub-figure,   the noise norm decrease as data heterogeneity alleviates.
Meanwhile,  the maximal noise norm is limited, i.e., $<0.012$, but the noise norm does not decrease to zero even in the i.i.d case due to the randomness of SGD.
Moreover, according to the right sub-figure,   the noise is not strongly correlated with the iterations when the data heterogeneity is mild (i.e., $\alpha \geq 0.4$),  and slightly decreases along the training iterations.
This  verifies Assumption \ref{assumption5}, where the neuron noise arises from data heterogeneity and remains independent of the number of iterations.

\section{Discussion and Future Work}

Through experimental and theoretical analysis,
we unravel the relationship between client and global modes in FL under varying data heterogeneity.
As data become more heterogeneous, finding solutions that optimize both the client and  FL objectives becomes increasingly challenging due to the reduction of overlapping solutions between client and global modes.
Meanwhile, the FL solutions obtained from different setups of data heterogeneity belong to distinct global modes, and are located in distinct basins of a joint region.
These solutions can be connected by a specific PolyChain while showing an error barrier along the linear interpolation.
This suggests that global modes in FL are not isolated but rather interconnected within a manifold.
Our analysis demonstrates that the connectivity error between global modes decreases with weakened data heterogeneity and wider trained models. 

The findings suggest potential directions for future work for the design of FL:
\textit{i}) exploring the solutions to the FL objective that maintain low error for both client and global test sets, i.e., the solutions to both client and FL objectives; 
\textit{ii}) investigating the effect of model depth and architecture in FL  based on mean-field theory;
\textit{iii}) understanding FL training dynamics by  the gap between random and pre-trained initialization (see a preliminary investigation in Appendix).

\bibliography{asilomar24}
\bibliographystyle{IEEEtran}

 \appendix

\section{Additional Experiments}

\subsection{Additional Experiments in client-mode connectivity}

\begin{figure}[th]
	\centering  
	\subfigbottomskip=2pt 
	\subfigcapskip=0pt 
 	\subfigure{
		\includegraphics[width=0.46\linewidth,height=90pt]{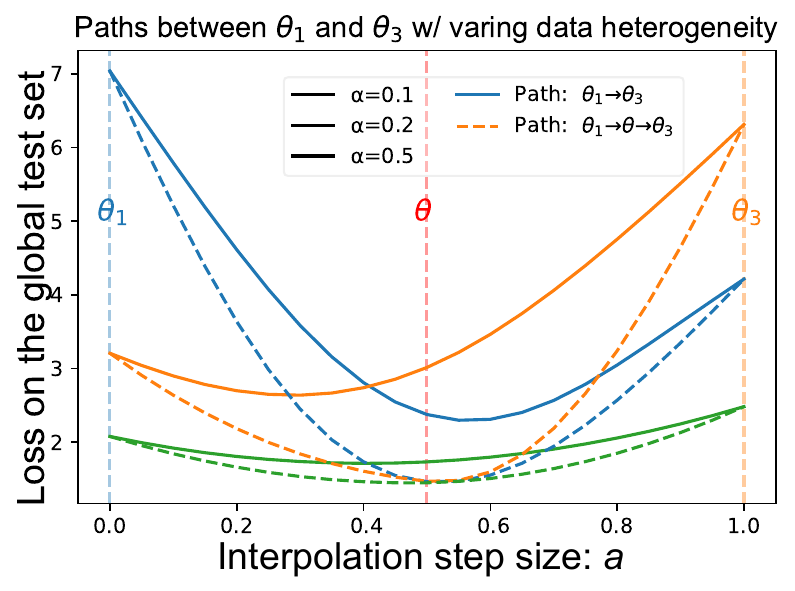}
  \label{fig:global_test_loss_data_heterogeneity_a}}
	\subfigure{
		\includegraphics[width=0.46\linewidth,height=90pt]{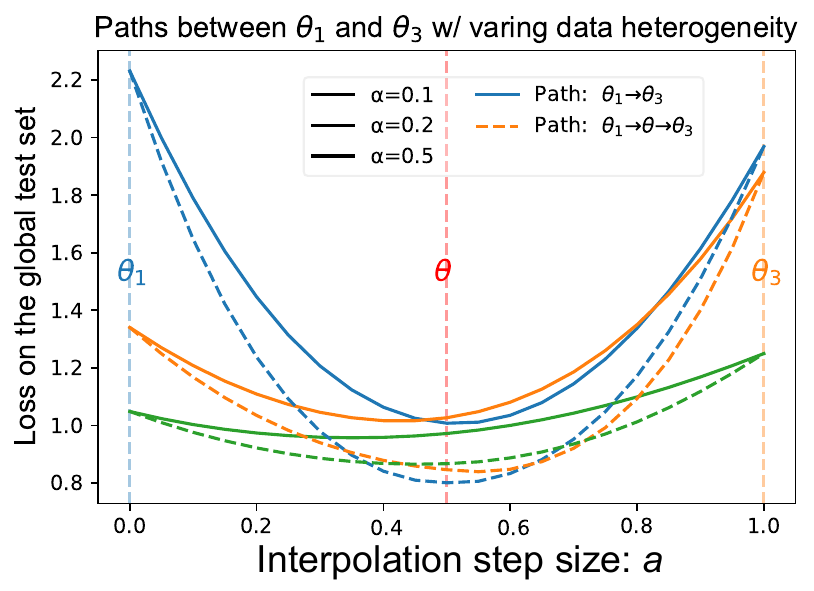}
    \label{fig:global_test_loss_data_heterogeneity_b}}
    \caption{Test loss of interpolated model along two given paths. The test set is the global test set of FL, where the VGG11 trained in FL is   from scratch (\textit{left}) and based on pre-trained VGG11 on Imagenet (\textit{right}).}
  \label{fig:global_test_loss_data_heterogeneity}
 
\end{figure}
\textbf{Testing the models on the global test set along the paths connecting client models.} 
 Figure \ref{fig:global_test_loss_data_heterogeneity} considers the same paths as that of Figure \ref{fig:Traversing_test_loss_VGG11_E1} (i.e., the linear path and the PolyChain path)  to connect two client models, and reports the loss of the models along both paths on the global test set.
The left and right sub-figures display the results of the global model trained from scratch and pre-training, respectively.
When data are more heterogeneous,  the loss gap between the two paths becomes greater.
For example, the gap of the case of  $\alpha=0.1$ is much larger than that of $\alpha=0.5$. 
Both sub-figures indicate that models along the PolyChain path exhibit good performance in the global test set when $a$ falls within the range of $[0.4, 0.6]$.
The sub-figure in the right-hand side of Figure \ref{fig:Traversing_test_loss_VGG11_E1} also shows that the models in the range  $ a \in [0.4, 0.6]$ exhibit good performance in the client test sets.
Upon comparing the left and right sub-figures, we observe a significant reduction in the loss gap between both paths due to the utilization of the pre-trained model.
Therefore, it appears that there are certain solutions to the FL objective (\ref{objective:FL}) that can maintain a low loss for both client test sets and the global test set, particularly when using a pre-trained model as the model initialization.
 It would be beneficial for future work to further explore how to find these solutions.

\begin{figure}[h]
	\centering  
	\subfigbottomskip=2pt 
	\subfigcapskip=0pt 
 	\subfigure{
		\includegraphics[width=0.46\linewidth]{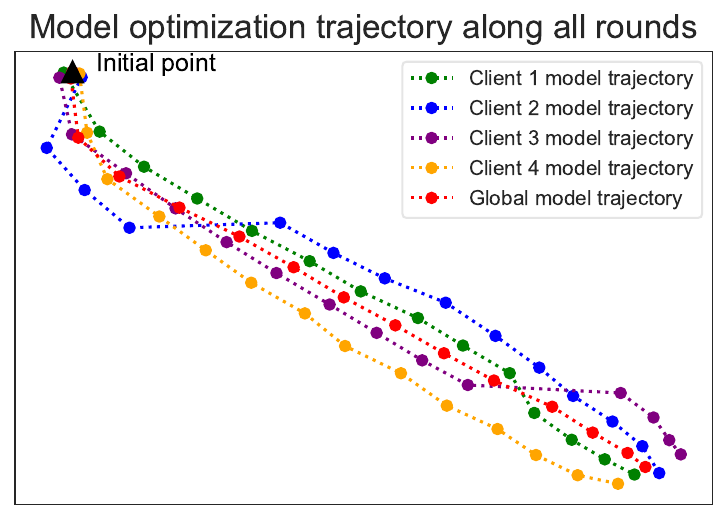}
  \label{fig:optimization_trajectory_CNN_E5_a}}
	\subfigure{
		\includegraphics[width=0.46\linewidth]{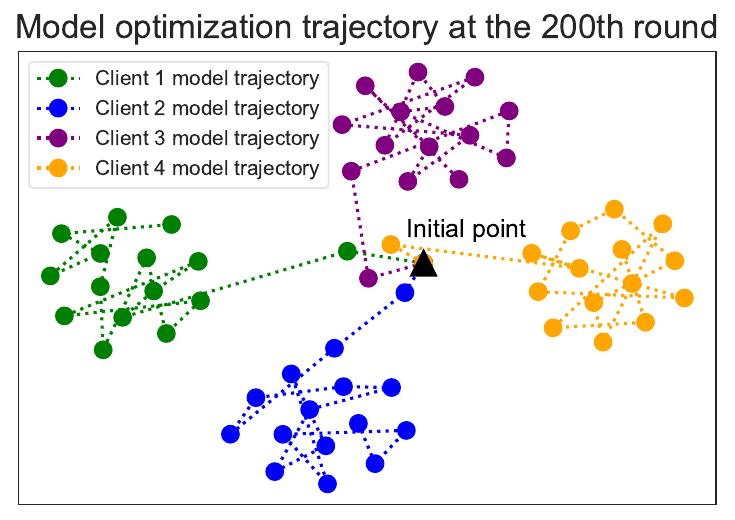}
    \label{fig:optimization_trajectory_CNN_E5_b}}
    \caption{Client models' optimization trajectory (shallow CNN with two convolutional layers) along the training rounds (\textit{left}) and along the local iteration within one round (\textit{right}).}
  \label{fig:optimization_trajectory_CNN_E5}
\end{figure}
\textbf{Optimization trajectory of shallow CNN.} 
In Figure \ref{fig:optimization_trajectory_CNN_E5}, we also visualize the optimization trajectory of client models to compare with the VGG11 trajectory illustrated in Figure \ref{fig:optimization_trajectory_VGG11_E1}. 
Similar to the observation of  Figure \ref{fig:optimization_trajectory_VGG11_E1}, the solutions found by client objectives vary under data heterogeneity even though they are close to each other in the parameter space.
Additionally, it is evident that during the training procedure, when utilizing shallow CNN, the distance between client solutions is regulated but is greater than that of VGG11.
This implies that data heterogeneity affects models with various dimension sizes differently, which is in line with the results of Theorem \ref{theorem:mode_connectivity_FL} and Figure \ref{fig:noise_iteration_comparision}.

\subsection{Additional Experiments in global-mode connectivity} 

\textbf{Connectivity landscape along the linear-interpolation path and the PolyChain path.}
\begin{figure}[th]
	\centering  
	\subfigbottomskip=2pt 
	\subfigcapskip=0pt 
 	\subfigure{
		\includegraphics[width=0.46\linewidth,height=90pt]{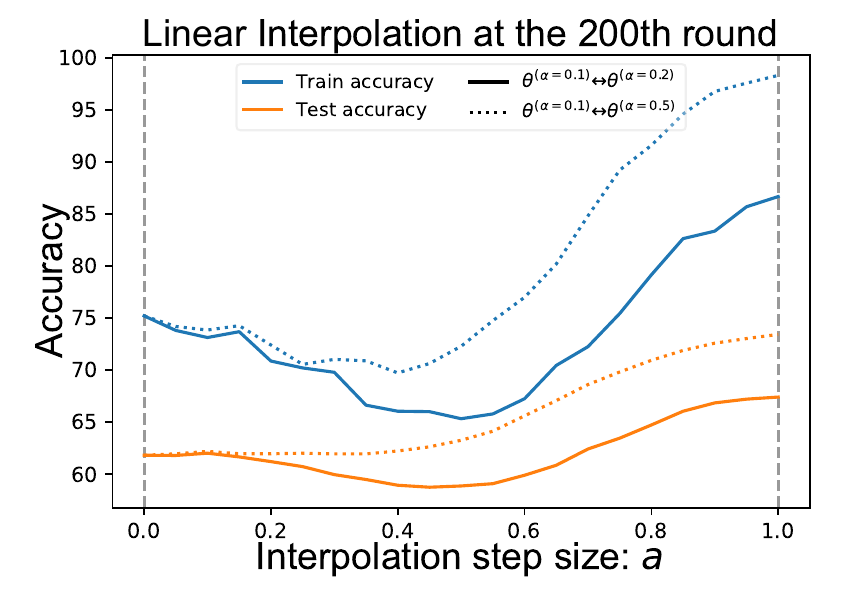}
  \label{fig:interpolation_a}}
	\subfigure{
		\includegraphics[width=0.46\linewidth,height=90pt]{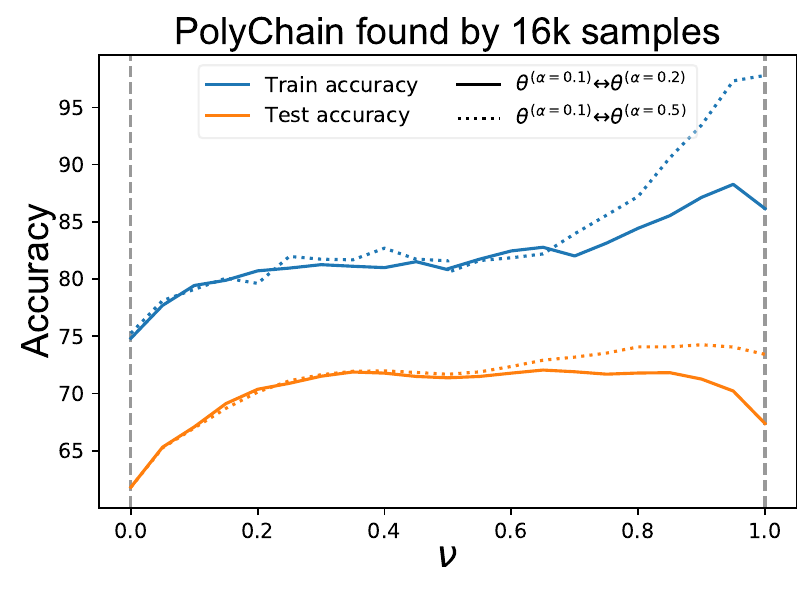}
    \label{fig:interpolation_b}}
    \caption{Linear-interpolation path (\textit{left}) and PolyChain path (\textit{right}) of global models.}
  \label{fig:global_models_interpolation}
\end{figure}
We take the classification accuracy as the metric and plot the connectivity landscape along the linear-interpolation path and the PolyChain path in Figure \ref{fig:global_models_interpolation}.
As additional data to Figure \ref{fig:global_model_acc_mode_connecitvity}, Figure \ref{fig:global_models_interpolation} displays comprehensive accuracy fluctuations along both paths.
As shown in the left sub-figure, when dealing with more heterogeneous data, such as in the case of $\alpha=0.1, 0.2$, both training and test barriers become higher, compared with $\alpha=0.1, 0.5$.
Meanwhile, since all global models are not far away from each other and are located in a common region, the maximal accuracy drop induced by the barriers is around $10\%$.
The right sub-figure shows that the PolyChain with one bend easily connects the global models from  $\alpha=0.1$ to $ 0.2$ and from  $\alpha=0.1$ to $ 0.5$, while all the models along the PolyChain keep no barrier.

\textbf{Neuron noise $\boldsymbol{n}_i$ insight into global-mode connectivity.}
\begin{figure}[th]
	\centering  
	\subfigbottomskip=2pt 
	\subfigcapskip=0pt 
 	\subfigure{
		\includegraphics[width=0.46\linewidth,height=90pt]{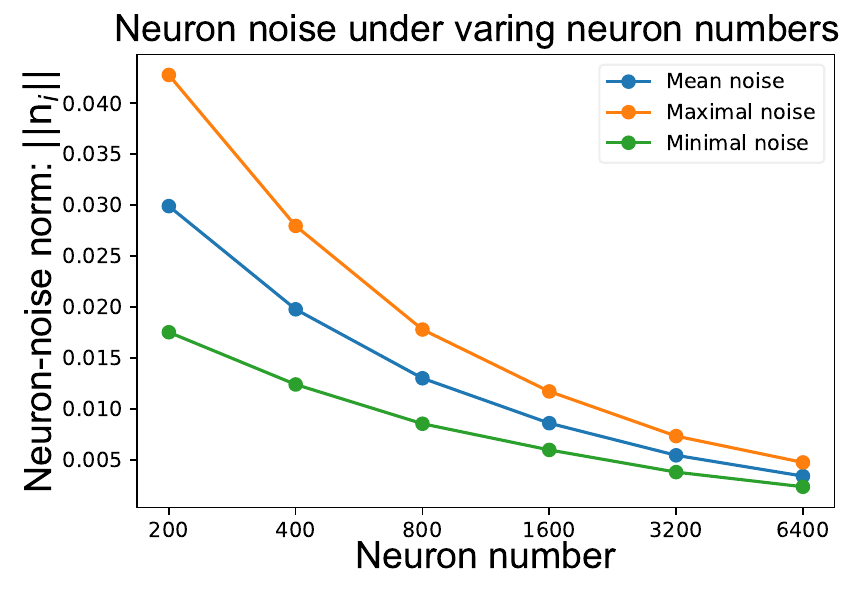}
  \label{fig:noise_iteration_comparision_a}}
	\subfigure{
		\includegraphics[width=0.46\linewidth,height=90pt]{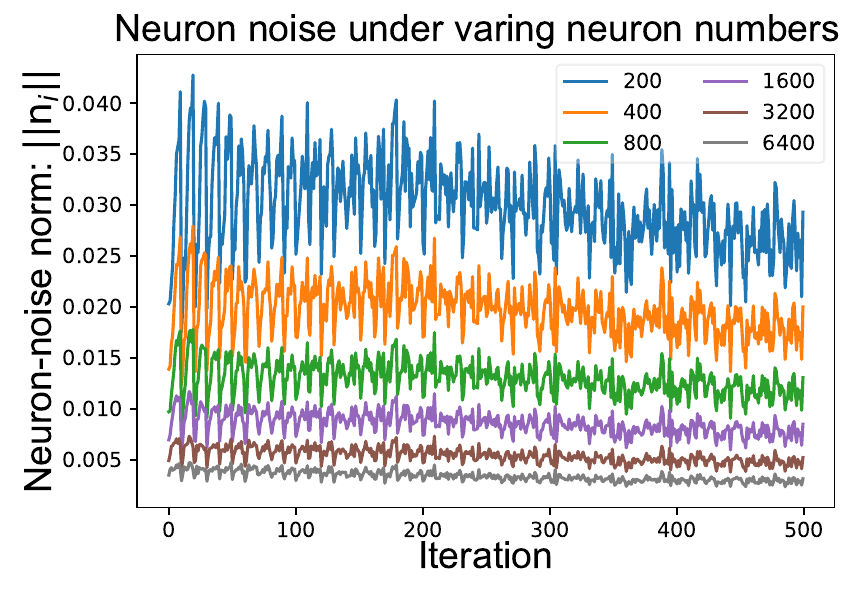}
    \label{fig:noise_iteration_comparision_b}}
    \caption{Neuron noise under varying neuron numbers  (\textit{left})  and its visualization along all the local iterations (\textit{right}).}
  \label{fig:noise_iteration_comparision}
\end{figure} 
We resort to the noise norm in (\ref{equation:noise}) to gain a better understanding of the relationship between the number of neurons ($N$) and the dropout error ($\varepsilon_{\mathrm{C}}$) and connectivity error ($\varepsilon_{\mathrm{D}}$).
As shown in Figure \ref{fig:noise_iteration_comparision}, the L2 norm of global-mode update noise is depicted under $\alpha=0.1$. 
This figure displays the noise norms for different neuron numbers and their visualization across all iterations. 
It also reports the mean, maximal, and minimal noise norms.
The left sub-figure demonstrates that as the number of neurons increases, the noise norm decreases.
When  $N=6400$, the gap between maximal and minimal noise norms becomes much smaller than that of $N=200$.
As shown in the right sub-figure, the noise levels are distinguishable for different numbers of neurons.
These results reveal that the wider models show smaller noise in their updates of FL, and then have lower dropout error ($\varepsilon_{\mathrm{C}}$) and connectivity error ($\varepsilon_{\mathrm{D}}$) in Figure \ref{fig:width_comparision}.

\section{Proof}
   The approximation in Lemma \ref{lemma:PDE_SGD} builds the connection from the nonlinear dynamics to the particle dynamics to the gradient descent to the  SGD under the data-heterogeneity noise. 
The approximation relies on the following model update:
   \begin{equation}  
    \boldsymbol{\theta}_i^{\tau+1}  =  (1-2\lambda s_\tau) \boldsymbol{\theta}_i^{\tau}+
             2 s_\tau\left(y^\tau-\hat{y}^{\tau}\right) \nabla_{\boldsymbol{\theta}_i} \sigma\left(\boldsymbol{x}_\tau ; \boldsymbol{\theta}_i^\tau\right)    +  \beta { \boldsymbol{n}}_i^\tau,
\end{equation}
where $\boldsymbol{\theta}_i$ denotes the $i$-th neuron parameter, $\lambda$ denotes the coefficient of regularizer in the objective (i.e., $\lambda=0$ without regularizers),  $s_\tau$ denotes the step size at the $\tau$-th iteration,  $\beta = \sqrt{2{s}^{\tau}\varepsilon/D}$ and $\beta { \boldsymbol{n}}_i^\tau$ denote the random noise at the $\tau$-th iteration.
 By Proposition 33, 35, 37 and 38 in \cite{SongMei2019meanfield}, when Assumptions \ref{assumption1} to \ref{assumption5} hold,  we have the following proposition:
\begin{proposition}\label{proposition:4approximation}
There exists a constant  $C$ (depending solely on the constants $C_i$ of assumptions 1-5), such that with probability at least $1-e^{-z^2}$, we have:
    \begin{align*}
\sup_{t \in[0, T]} & \left|\mathcal{L}_N\left(\overline{\boldsymbol{\theta}}^t\right)-\mathcal{L}\left(\rho_t\right)\right| \\  
\leq & C e^{C T} \frac{\sqrt{h(\alpha)}}{\sqrt{N}}[\sqrt{\log (N T)}+z], \\
\sup _{t \in[0, T]} & \left|\mathcal{L}_N\left(\underline{\boldsymbol{\theta}}^t\right)-\mathcal{L}_N\left(\overline{\boldsymbol{\theta}}^t\right)\right| \\
\leq & C e^{C T} \frac{1}{\sqrt{N}}[\sqrt{\log (N T)}+z], \\
\sup_{\tau \in[0, T / \varepsilon] \cap \mathbb{N}}& \left|\mathcal{L}_N\left(\underline{\boldsymbol{\theta}}^{\tau \varepsilon}\right)-\mathcal{L}_N\left(\tilde{\boldsymbol{\theta}}^\tau\right)\right| \\ 
\leq &  C e^{C T}[\sqrt{\log (N(T / \varepsilon \vee 1)}+z] \sqrt{\varepsilon h(\alpha)}, \\
\sup _{\tau \in[0, T / \varepsilon] \cap \mathbb{N}}& \left|\mathcal{L}_N\left(\tilde{\boldsymbol{\theta}}^\tau\right)-\mathcal{L}_N\left(\boldsymbol{\theta}^\tau\right)\right| \\ 
\leq &  C e^{C T} \sqrt{T \varepsilon}[\sqrt{D+\log N}+z],
\end{align*}
where $\overline{\boldsymbol{\theta}}, \underline{\boldsymbol{\theta}}^t, \tilde{\boldsymbol{\theta}}^\tau$, and $ {\boldsymbol{\theta}}^\tau$ denote the solutions of the nonlinear dynamics, particle dynamics, gradient descent and  SGD, respectively.
\end{proposition}
In Proposition \ref{proposition:4approximation}, the main distinction from its counterpart in \cite{SongMei2019meanfield} is that it accounts for the impact of data heterogeneity in the model update, i.e.,   $\sqrt{h(\alpha)}$ is not embraced into the constant $C$. 
The proof for Proposition \ref{proposition:4approximation} aligns with the proof of its counterpart in \cite{SongMei2019meanfield}, and see \cite{SongMei2019meanfield} for the details.

\begin{lemma}\label{lemma:distance_bound_PDE}
  (From Lemma 31 in \cite{SongMei2019meanfield}).  There exists a constant $K$, such that with probability at least $1-e^{-z^2}$,
$$
\begin{aligned}
\sup _{i \leq N} \sup _{\tau \in[0, T / \varepsilon] \cap \mathbb{N}} \sup _{u \in[0, \varepsilon]} & \left\|\overline{\boldsymbol{\theta}}_i^{\tau \varepsilon+u}-\overline{\boldsymbol{\theta}}_i^{\tau \varepsilon}\right\|_2
\\ \leq & C e^{C T}[\sqrt{\log (N(T / \varepsilon \vee 1))}+z] \sqrt{\varepsilon},
\end{aligned}
$$
 where $\overline{\boldsymbol{\theta}}_i^{\tau \varepsilon+u}$ and $\overline{\boldsymbol{\theta}}_i^{\tau \varepsilon}$ denote the solutions of the PDE in Lemma \ref{lemma:PDE_SGD} at time $\tau\varepsilon+u$ and $\tau \varepsilon$, respectively.
\end{lemma}

\subsection{Proof of Theorem \ref{theorem:dropout_stability_FL}} 
\begin{proof}
In the following proof, we take $C$ to accommodate all constants that depend solely on the constants $C_i$ of assumptions 1-4.
According to Definition \ref{def:dropout_stability}, we have:
\begin{equation}
\begin{aligned}
  \sup_{\tau\in[0,T/\varepsilon]\cap \mathbb{N}} &  \left|\mathcal{L}_N(\boldsymbol{\theta}^\tau)-\mathcal{L}_{|\mathcal{A}|}\left(\boldsymbol{\theta}_{\mathrm{S}}^\tau\right)\right| \\
  \leq  & \left|\mathcal{L}_N(\boldsymbol{\theta}^\tau)-\mathcal{L}\left(\rho_{\tau\varepsilon}\right)\right| 
     + \left|\mathcal{L}_{|\mathcal{A}|}(\boldsymbol{\theta}_{\mathrm{S}}^\tau)-\mathcal{L}\left(\rho_{\tau\varepsilon}\right)\right| 
     \\ 
     \leq   & \left|\mathcal{L}_N(\boldsymbol{\theta}^\tau)-\mathcal{L}\left(\rho_{\tau\varepsilon}\right)\right| 
     +\left|\mathcal{L}_{|\mathcal{A}|}(\boldsymbol{\theta}_{\mathrm{S}}^\tau)-\mathcal{L}_{|\mathcal{A}|}(\Bar{\boldsymbol{\theta}}_{\mathrm{S}}^{\tau\varepsilon})\right| \\
      &+ \left| \mathcal{L}_{|\mathcal{A}|}(\Bar{\boldsymbol{\theta}}_{\mathrm{S}}^{\tau\varepsilon})-\mathcal{L}\left(\rho_{\tau\varepsilon}\right)\right|,
\end{aligned}
\end{equation}
where $\Bar{\boldsymbol{\theta}}_{\mathrm{S}}^{\tau\varepsilon}$ denotes the dropout model  containing the first $\mathcal{|A|}$ non-zero elements (i.e., non-dropout elements) of $\Bar{\boldsymbol{\theta}}^{\tau\varepsilon}$.
Here, $\Bar{\boldsymbol{\theta}}^{\tau\varepsilon}$ is the PDE solution of the distributional dynamic at time $\tau\varepsilon$. 

\textbf{For the first term,}    we  take   Lemma \ref{lemma:PDE_SGD} and directly obtain:
\begin{equation}\label{proof:first_term}
    \begin{aligned}
    \sup_{\tau\in[0,T/\varepsilon]\cap \mathbb{N}} &\left|\mathcal{L}_N(\boldsymbol{\theta}^\tau)-\mathcal{L}\left(\rho_{\tau\varepsilon}\right)\right|    \\
\leq &C e^{C T } \sqrt{h(\alpha)}\left(\frac{\sqrt{\log  N }+z}{\sqrt{N}}+\sqrt{\varepsilon}(\sqrt{D+\log  N }+z)\right).
\end{aligned}
\end{equation}

  \textbf{For the second term},  we compute its upper bound   based on the maximal output gap between the $i$-th neuron of $ \boldsymbol{\theta}_{\mathrm{S}}^\tau $ and the $j$-th neuron of $ \overline{\boldsymbol{\theta}}_{\mathrm{S}}^{\tau \varepsilon} $, which can be represented as: 
 \begin{equation}
\begin{aligned}
& \left|\mathcal{L}_{|\mathcal{A}|}\left(\boldsymbol{\theta}_{\mathrm{S}}^\tau\right)-\mathcal{L}_{|\mathcal{A}|} (\overline{\boldsymbol{\theta}}_{\mathrm{S}}^{\tau \varepsilon} )\right|    \\
 \leq & 2 \max _{i \in \mathcal{A}}\left|  \mathbb{E}\left\{y \sigma\left(\boldsymbol{x}, \boldsymbol{\theta}_i^k\right)\right\}-  \mathbb{E}\left\{y \sigma (\boldsymbol{x}, \overline{\boldsymbol{\theta}}_i^{\tau \varepsilon} )\right\}\right| \\
&  +\max _{i, j \in \mathcal{A}}\left|    \mathbb{E}\left\{\sigma\left(\boldsymbol{x}, \boldsymbol{\theta}_i^k\right) \sigma\left(\boldsymbol{x}, \boldsymbol{\theta}_j^k\right)\right\}-   \mathbb{E}\left\{\sigma (\boldsymbol{x}, \overline{\boldsymbol{\theta}}_i^{\tau \varepsilon} t) \sigma (\boldsymbol{x}, \overline{\boldsymbol{\theta}}_j^{\tau \varepsilon} )\right\}\right| \\
 \leq&   4C_2  \max _{i \in \mathcal{[|A|]}}\left\|\boldsymbol{\theta}_i^k-\overline{\boldsymbol{\theta}}_i^{\tau \varepsilon}\right\|_2 \\
 \leq & 4C_2 \max _{i \in[N]}\left\|\boldsymbol{\theta}_i^k-\overline{\boldsymbol{\theta}}_i^{\tau \varepsilon}\right\|_2,
\end{aligned}
 \end{equation}
where   the second inequality follows that $y, \sigma(\cdot)$ and the gradient of $\sigma(\cdot)$ are bounded. 
 
Furthermore, we take the sum of all inequalities in Proposition \ref{proposition:4approximation} and obtain:
 \begin{equation}
\begin{aligned}
&\sup _{\tau \in[T / \varepsilon]} \max _{i \in[N]}\left\|\boldsymbol{\theta}_i^k-\overline{\boldsymbol{\theta}}_i^{\tau \varepsilon}\right\|_2 \\
\leq &C  e^{C T}\sqrt{h(\alpha)}\left(\frac{\sqrt{\log N}+z}{\sqrt{N}}+\sqrt{\varepsilon}(\sqrt{D+\log  N }+z)\right),
\end{aligned}
 \end{equation}
  with probability at least $1-e^{-z^2}$.
  Consequently, we have: with probability at least $1-e^{-z^2}$,
 \begin{equation}\label{proof:second_term}
\begin{aligned}
&\sup _{\tau \in[T /\varepsilon]}\left|\mathcal{L}_{|\mathcal{A}|}\left(\boldsymbol{\theta}_{\mathrm{S}}^\tau\right)-\mathcal{L}_{|\mathcal{A}|}\left(\overline{\boldsymbol{\theta}}_{\mathrm{S}}^{\tau \varepsilon}\right)\right| \\
&\leq C e^{C  T}\sqrt{h(\alpha)}\left(\frac{\sqrt{\log N}+z}{\sqrt{N}}+\sqrt{\varepsilon}(\sqrt{D+\log  N }+z)\right).
\end{aligned}
 \end{equation}

\textbf{For the third term}, we follow the triangle inequality and have:
\begin{equation}\label{proof:third_term_triangle_inequality}
    \begin{aligned}
    \sup_{\tau \in [0,T/\varepsilon] \cap \mathbb{N}}   & \left|\mathcal{L}_{| \mathcal{A}|}\left(\overline{\boldsymbol{\theta}}_{\mathrm{S}}^{\tau\varepsilon}\right)-\mathcal{L}\left(\rho_{\tau\varepsilon}\right)\right| \\
        \leq &\left|\mathcal{L}_{|\mathcal{A}|}\left(\overline{\boldsymbol{\theta}}_{\mathrm{S}}^{\tau\varepsilon}\right)-\mathbb{E}_{\rho_0}\left\{\mathcal{L}_{|\mathcal{A}|}\left(\overline{\boldsymbol{\theta}}_{\mathrm{S}}^{\tau\varepsilon}\right)\right\}\right|\\
       & +\left|\mathbb{E}_{\rho_0}\left\{\mathcal{L}_{| \mathcal{A}|}\left(\overline{\boldsymbol{\theta}}_{\mathrm{S}}^{\tau\varepsilon}\right)\right\}-\mathcal{L}\left(\rho_{k \mathrm{a}}\right)\right|,
    \end{aligned}
\end{equation}
where $\mathbb{E}_{\rho_0}\{\cdot\}$  takes the expectation on $\boldsymbol{\theta}_i^0 \sim \rho_0$. 
For the  first term of (\ref{proof:third_term_triangle_inequality}),
we consider  the mean-square loss in this work and have:
\begin{equation}
    \begin{aligned}
\mathcal{L}_{|\mathcal{A}|}\left(\boldsymbol{\theta}_{\mathrm{S}}\right)=\mathbb{E}_{(x, y) \sim \mathbb{P}}\left\{\left(y-\frac{1}{|\mathcal{A}|} \sum_{i=1}^{| \mathcal{A}|} \sigma\left(\boldsymbol{x}, \boldsymbol{\theta}_i\right)\right)^2\right\},
    \end{aligned}
\end{equation}
where  $\mathbb{E}_{(x, y)}$  takes the expectation on   $(\boldsymbol{x}, y) \sim \mathbb{P}$. 
According to Lemma \ref{lemma:PDE_SGD}, we have $\left\{\boldsymbol{\theta}_i^{\tau\varepsilon}\right\}_{i=1}^{| \mathcal{A}|} \sim \rho_{\tau\varepsilon}$ under the PDE approximation, and  rewrite the second term of (\ref{proof:third_term_triangle_inequality}) as:
\begin{equation} \label{proof:third_term_2}
    \begin{aligned}
& \sup_{\tau \in [0,T/\varepsilon] \cap \mathbb{N}}\left|\mathbb{E}_{\rho_0}\left\{\mathcal{L}_{| \mathcal{A}|}\left(\overline{\boldsymbol{\theta}}_{\mathrm{S}}^{\tau\varepsilon}\right)\right\}-\mathcal{L}\left(\rho_{k \mathrm{a}}\right)\right| \\
= &\sup _{\tau \in [0,T/\varepsilon]} \frac{1}{|\mathcal{A}|}\left|\int \mathbb{E}_{(x, y)}\left\{\left(\sigma(\boldsymbol{x}, \boldsymbol{\theta})\right)^2\right\} \rho_{\tau\varepsilon}(\mathrm{d} \boldsymbol{\theta}) \right. \\
  &  - \left.  \int  \mathbb{E}_{(x, y)}\left\{\sigma\left(\boldsymbol{x}, \boldsymbol{\theta}_1\right) \sigma\left(\boldsymbol{x}, \boldsymbol{\theta}_2\right)\right\} \rho_{\tau\varepsilon}\left(\mathrm{d} \boldsymbol{\theta}_1\right) \rho_{\tau\varepsilon}\left(\mathrm{d} \boldsymbol{\theta}_2\right)\right| \\
\leq  &   \frac{C}{|\mathcal{A}|},
    \end{aligned}
\end{equation}
 where the inequality is because  the activation function $\sigma(\cdot)$ is bounded by $C_2$ according to Assumption \ref{assumption2}.
That is, 
denoting by $\boldsymbol{\theta}$ and $\boldsymbol{\theta}^{\prime}$ be two parameters that differ only in the $i$-th component, i.e., $\boldsymbol{\theta}=\left(\boldsymbol{\theta}_1, \ldots, \boldsymbol{\theta}_{i,}, \ldots, \boldsymbol{\theta}_{| \mathcal{A}|}\right)$ and $\boldsymbol{\theta}^{\prime}=$ $\left(\boldsymbol{\theta}_1, \ldots, \boldsymbol{\theta}_i^{\prime}, \ldots, \boldsymbol{\theta}_{| \mathcal{A}|}\right)$, we have:
$\left|\mathcal{L}_{|\mathcal{A}|}(\boldsymbol{\theta})-\mathcal{L}_{|\mathcal{A}|}\left(\boldsymbol{\theta}^{\prime}\right)\right| \leq {C}/{|\mathcal{A}|}$.

With McDiarmid's inequality,  we have:
\begin{equation}
    \begin{aligned}
\mathbb{P}\left(\left|\mathcal{L}_{| \mathcal{A}|}\left(\bar{\boldsymbol{\theta}}_{\rm S}^t\right)-\mathbb{E}_{\rho_0}\left\{\mathcal{L}_{| \mathcal{A}|}\left(\bar{\boldsymbol{\theta}}_{\rm S}^t\right)\right\}\right|>\delta\right) \leq \exp \left(-\frac{|\mathcal{A}| \delta^2}{C}\right).
    \end{aligned}
\end{equation}

Furthermore, we have the following increment bound for time $t, h \geq 0$:
\begin{equation}
\begin{aligned}
& \left| | \mathcal{L}_{| \mathcal{A}|}\left(\overline{\boldsymbol{\theta}}_{\rm S}^{t+h}\right)-\mathbb{E}_{\rho_0} \left\{ \mathcal{L}_{| \mathcal{A}|}\left(\overline{\boldsymbol{\theta}}_{\rm S}^{t+h}\right)\right\}| \right. \\
& - \left. | \mathcal{L}_{| \mathcal{A}|}\left(\overline{\boldsymbol{\theta}}_{\rm S}^t\right)-\mathbb{E}_{\rho_0} \left\{ \mathcal{L}_{| \mathcal{A}|}\left(\overline{\boldsymbol{\theta}}_{\rm S}^t\right)\right\}|\right| \\
\leq & \left|\mathcal{L}_{| \mathcal{A}|}\left(\overline{\boldsymbol{\theta}}_{\rm S}^{t+h}\right)-\mathcal{L}_{| \mathcal{A}|}\left(\overline{\boldsymbol{\theta}}_{\rm S}^t\right)\right| \\ 
& +\left|\mathbb{E}_{\rho_0} \left\{ \mathcal{L}_{| \mathcal{A}|}\left(\overline{\boldsymbol{\theta}}_{\rm S}^{t+h}\right)\right\}-\mathbb{E}_{\rho_0} \left\{ \mathcal{L}_{| \mathcal{A}|}\left(\overline{\boldsymbol{\theta}}_{\rm S}^t\right)\right\}\right| \\
\leq & C\left[\sup _{i \in[N]}\left\|\overline{\boldsymbol{\theta}}_i^{t+h}-\overline{\boldsymbol{\theta}}_i^t\right\|_2+\mathbb{E}\left[\left\|\overline{\boldsymbol{\theta}}_j^{t+h}-\overline{\boldsymbol{\theta}}_j^t\right\|_2\right]\right] .
\end{aligned}
\end{equation}

Using Lemma \ref{lemma:distance_bound_PDE}, we get:
\begin{equation}
\begin{aligned}
  &  \sup _{\tau \in[0, T / \varepsilon] \cap \mathbb{N}} \sup _{u \in[0, \varepsilon]}  \left|| \mathcal{L}_{| \mathcal{A}|}\left(\overline{\boldsymbol{\theta}}_{\rm S}^{\tau \varepsilon+u}\right)-\mathbb{E}_{\rho_0} \left\{ \mathcal{L}_{| \mathcal{A}|}\left(\overline{\boldsymbol{\theta}}_{\rm S}^{\tau \varepsilon+u}\right)\right\}|  \right. \\
 & \quad \quad \quad \quad \quad\quad \quad \quad-\left.| \mathcal{L}_{| \mathcal{A}|}\left(\overline{\boldsymbol{\theta}}_{\rm S}^{\tau \varepsilon}\right)-\mathbb{E}_{\rho_0} \left\{\mathcal{L}_{| \mathcal{A}|}\left(\overline{\boldsymbol{\theta}}_{\rm S}^{\tau \varepsilon}\right)\right\}|\right| \\
\leq    &    C e^{C T}[\sqrt{\log |\mathcal{A}|(T / \varepsilon \vee 1)}+z] \sqrt{\varepsilon h(\alpha)}, \\
\end{aligned}
\end{equation}
with probability at least $1-e^{-z^2}$.

Here, we take a union bound over $s \in \varepsilon\{0,1, \ldots,\lfloor T / \varepsilon\rfloor\}$ and bound  the variation inside the grid intervals.
Then,  with McDiarmid’s inequality, we get:
\begin{equation}\label{proof:third_term_1}
\begin{aligned}
& \mathbb{P}\left(\sup _{t \in[0, T]} |\mathcal{L}_{| \mathcal{A}|}\left(\overline{\boldsymbol{\theta}}_{\rm S}^t\right)-\mathbb{E}_{\rho_0} \mathcal{L}_{| \mathcal{A}|}\left(\overline{\boldsymbol{\theta}}_{\rm S}^t\right) | \right.\\
& \quad \quad \geq \left. \delta+C e^{C T}[\sqrt{\log \mathcal{|A|}(T / \varepsilon \vee 1)}+z] \sqrt{\varepsilon h(\alpha)}\right) \\
\leq & (T / \varepsilon) \exp \left\{-\mathcal{|A|} \delta^2 / C\right\}+e^{-z^2}
\end{aligned}
\end{equation}
We take $\varepsilon=1 / \mathcal{|A|}$ and $\delta=C[\sqrt{\log (\mathcal{|A|} T)}+z] / \sqrt{\mathcal{|A|}}$ in (\ref{proof:third_term_1}) and combine it with (\ref{proof:third_term_1}), and get:
\begin{equation}\label{proof:third_term}
    \begin{aligned}
     &\sup_{\tau \in [0,T/\varepsilon] \cap \mathbb{N}}    \left|\mathcal{L}_{| \mathcal{A}|}\left(\overline{\boldsymbol{\theta}}_{\mathrm{S}}^{\tau\varepsilon}\right)-\mathcal{L}\left(\rho_{\tau\varepsilon}\right)\right| \\
     \leq & C e^{C T}\frac{(\sqrt{\log |\mathcal{A}|}+z)\sqrt{h(\alpha)}}{\sqrt{|\mathcal{A}|}}
    \end{aligned}
\end{equation}
with probability at least $1-e^{-z^2}$.

we  summarize (\ref{proof:first_term}), (\ref{proof:second_term}) and (\ref{proof:third_term}) and have:
\begin{equation}
\begin{aligned}
 &   \sup_{\tau\in[0,T/\varepsilon]\cap \mathbb{N}}  \left|\mathcal{L}_N(\boldsymbol{\theta}^\tau)-\mathcal{L}_{|\mathcal{A}|}\left(\boldsymbol{\theta}_{\mathrm{S}}^\tau\right)\right| \\
\leq  &  C e^{C T}\sqrt{h(\alpha)}\left(\frac{\sqrt{\log |\mathcal{A}|}+z}{\sqrt{|\mathcal{A}|}}+\sqrt{\alpha}(\sqrt{D+\log N}+z)\right),
\end{aligned}
\end{equation}
with probability at least $1-e^{-z^2}$.
Therefore, the global model $\boldsymbol{\theta}$ obtained from the FL under  data heterogeneity $h(\alpha)$ is $\varepsilon_D$-dropout
stable with $\varepsilon_D$ equal to $C e^{C T}\sqrt{h(\alpha)}[({\sqrt{\log |\mathcal{A}|}+z})/{\sqrt{|\mathcal{A}|}}+\sqrt{\alpha}(\sqrt{D+\log N}+z)]$.
    
\end{proof}

\subsection{Proof of Theorem \ref{theorem:mode_connectivity_FL}} 
The proof of Theorem \ref{theorem:mode_connectivity_FL} is obtained by combining Theorem \ref{theorem:dropout_stability_FL} with the following lemma, which is based on \cite{Kuditipudi2019Explaining_Landscape,shevchenko20aLandscape}.

\begin{lemma}
    (Dropout networks obtained from FL under varying data heterogeneity have connectivity). Consider a two-layer neural network with $N$ neurons, as in (\ref{equation:twolayer_network_function}).
    Given $\mathcal{A}=[N / 2]$, let $\boldsymbol{\theta}$ and $\boldsymbol{\theta}^{\prime}$ be  solved under data heterogeneity $h(\alpha)$ and $h(\alpha^\prime)$, which  have dropout stability as in Definition \ref{def:dropout_stability}. 
    Then, $\boldsymbol{\theta}$ and $\boldsymbol{\theta}^{\prime}$ are $\varepsilon$-connected as in Definition \ref{def:model_connectivity}. 
    Furthermore, the path connecting $\boldsymbol{\theta}$ with $\boldsymbol{\theta}^{\prime}$ consists of 7 line segments.
\end{lemma}
 \begin{proof}
To consider the effect of the dropout stability on the given models, we   represent the model parameters as: 
$$
\boldsymbol{\theta}=\left((a_1, \boldsymbol{\theta}_1),(a_2, \boldsymbol{\theta}_2), \ldots,(a_N, \boldsymbol{\theta}_N)\right),
$$
$$
\boldsymbol{\theta}^{\prime}=\left((a_1^{\prime}, \boldsymbol{\theta}_1^{\prime}),(a_2^{\prime}, \boldsymbol{\theta}_2^{\prime}), \ldots,(a_N^{\prime}, \boldsymbol{\theta}_N^{\prime})\right),
$$ 
 where $a_i, \forall i \in [N]$ is the rescaling factor of dropout.
 That is, $a_i=0$ if the model components ($i \notin \mathcal{A}$) perform dropout, and $a_i=2$ if the model components ($i \in \mathcal{A}$) do not perform dropout.
 
 We consider the piecewise linear path in parameter space that connects $\boldsymbol{\theta}$ to $\boldsymbol{\theta}^{\prime}$  as one path of
mode connectivity   for over-parameterized neural networks.
Namely, when $N$ is even, we can connect $\boldsymbol{\theta}$ to $\boldsymbol{\theta}^{\prime}$ with the following way:
$$
\begin{aligned}
&\boldsymbol{\theta}=\left((a_1, \boldsymbol{\theta}_1),\ldots,(a_{N/2}, \boldsymbol{\theta}_{N / 2}), \ldots, (a_N, \boldsymbol{\theta}_N)\right),\\
& \boldsymbol{\theta}_1=\left((2, \boldsymbol{\theta}_1), \ldots,(2, \boldsymbol{\theta}_{N / 2}),(0, \boldsymbol{\theta}_{N / 2+1}), \ldots,(0, \boldsymbol{\theta}_N)\right), \\
& \boldsymbol{\theta}_2=\left((2, \boldsymbol{\theta}_1),\ldots,(2, \boldsymbol{\theta}_{N / 2}),\left(0, \boldsymbol{\theta}_1^{\prime}\right), \ldots,(0, \boldsymbol{\theta}_{N / 2}^{\prime})\right), \\
& \boldsymbol{\theta}_3=\left((0, \boldsymbol{\theta}_1), \ldots,(0, \boldsymbol{\theta}_{N / 2}),\left(2, \boldsymbol{\theta}_1^{\prime}\right), \ldots,(2, \boldsymbol{\theta}_{N / 2}^{\prime})\right), \\
& \boldsymbol{\theta}_4=\left(\left(0, \boldsymbol{\theta}_1^{\prime}\right), \ldots,(0, \boldsymbol{\theta}_{N / 2}^{\prime}),(2, \boldsymbol{\theta}_1^{\prime}), \ldots,(2, \boldsymbol{\theta}_{N / 2}^{\prime})\right), \\
& \boldsymbol{\theta}_5=\left((2, \boldsymbol{\theta}_1^{\prime}),\ldots,(2, \boldsymbol{\theta}_{N / 2}^{\prime}),\left(0, \boldsymbol{\theta}_1^{\prime}\right),\ldots,(0, \boldsymbol{\theta}_{N / 2}^{\prime})\right), \\
& \boldsymbol{\theta}_6=\left((2, \boldsymbol{\theta}_1^{\prime}), \ldots,(2, \boldsymbol{\theta}_{N / 2}^{\prime}),(0, \boldsymbol{\theta}_{N / 2+1}^{\prime}), \ldots,(0, \boldsymbol{\theta}_N^{\prime})\right),\\
& \boldsymbol{\theta}^{\prime}=\left((a_1^{\prime}, \boldsymbol{\theta}_1^{\prime}),  \ldots,(a_{N/2}^{\prime}, \boldsymbol{\theta}_{N / 2}^{\prime}), \ldots,(a_N^{\prime}, \boldsymbol{\theta}_N^{\prime})\right).
\end{aligned}
$$\

Firstly, for the path between $\boldsymbol{\theta}$ to $\boldsymbol{\theta}_1$, since $\boldsymbol{\theta}$ is $\varepsilon$-dropout stable, we have that $\mathcal{L}_N\left(\boldsymbol{\theta}_1\right) \leq \mathcal{L}_N(\boldsymbol{\theta})+\varepsilon_{D}$.  
Similarly, the loss along the path that connects $\boldsymbol{\theta}_6$ to $\boldsymbol{\theta}^{\prime}$ is also upper bounded by $\mathcal{L}_N\left(\boldsymbol{\theta}_6\right) \leq 
 \mathcal{L}_N\left(\boldsymbol{\theta}^{\prime}\right)+\varepsilon_{D}^\prime$.
 
Then,  for the path between $\boldsymbol{\theta}_1$ to $\boldsymbol{\theta}_2$,   only $\boldsymbol{\theta}_i$ with the corresponding $a_i=0$ is changed to be $\boldsymbol{\theta}^\prime_i$. 
Thus, the loss along this path does not change; i.e., $ \mathcal{L}_N(\boldsymbol{\theta}_1) = \mathcal{L}_N(\boldsymbol{\theta}_2)$. 
Similarly, the loss does not change along the path  between   $\boldsymbol{\theta}_3$ to $\boldsymbol{\theta}_4$ and the path  between $\boldsymbol{\theta}_5$ to $\boldsymbol{\theta}_6$: $ \mathcal{L}_N(\boldsymbol{\theta}_3) = \mathcal{L}_N(\boldsymbol{\theta}_4)$ and $ \mathcal{L}_N(\boldsymbol{\theta}_5) = \mathcal{L}_N(\boldsymbol{\theta}_6)$.

Next, for the path between $\boldsymbol{\theta}_4$ to $\boldsymbol{\theta}_5$,  the  two  subnetworks are equal, such that the loss of their interpolated models along this path does not change;  i.e., $ \mathcal{L}_N(\boldsymbol{\theta}_4) = \mathcal{L}_N(\boldsymbol{\theta}_5)$.

Finally, based on the above analysis, we have $\mathcal{L}_N\left(\boldsymbol{\theta}_2\right) \leq \mathcal{L}_N(\boldsymbol{\theta})+\varepsilon_{D}$ and $\mathcal{L}_N\left(\boldsymbol{\theta}_3\right) \leq \mathcal{L}_N(\boldsymbol{\theta}^\prime)+\varepsilon_{D}^\prime$.
For  the path between $\boldsymbol{\theta}_2$ to $\boldsymbol{\theta}_3$, since the loss is convex in the weights of the last layer, the loss along this path is upper bounded by $\max \left(\mathcal{L}_N(\boldsymbol{\theta}_2), \mathcal{L}_N\left(\boldsymbol{\theta}_{3}\right)\right) \leq \max \left(\mathcal{L}_N(\boldsymbol{\theta})+\varepsilon_D, \mathcal{L}_N\left(\boldsymbol{\theta}^{\prime}\right)+\varepsilon_{D}^\prime\right)$. 

When $N$ is odd, we can take a similar analysis to that of even $N$.  
The differences are that (i) the $\lceil N / 2\rceil$-th parameter of $\boldsymbol{\theta}_1, \boldsymbol{\theta}_2$ and $\boldsymbol{\theta}_3$ is $\left(0, \boldsymbol{\theta}_{N / 2}\right)$ and the $\lceil N / 2\rceil$-th parameter of $\boldsymbol{\theta}_4, \boldsymbol{\theta}_5$ and $\boldsymbol{\theta}_6$ is $\left(0, \boldsymbol{\theta}_{N / 2}^{\prime}\right)$, and (ii) the constant rescalling factor $a_i=2$  is replaced by $N /\lfloor N / 2\rfloor$.

Therefore, the loss along the above path between $\boldsymbol{\theta}$ and $\boldsymbol{\theta}^{\prime}$ is upper bounded by  $\max \left(\mathcal{L}_N(\boldsymbol{\theta}), \mathcal{L}_N\left(\boldsymbol{\theta}^{\prime}\right)\right)+\varepsilon_C$ with $\varepsilon_C$ equal to $\max(\varepsilon_{D}, \varepsilon_{D}^\prime)$.

 \end{proof}

\subsection{Proof of global-mode update (\ref{equation:update_FedAvg_simple})}
\label{proof:update_FedAvg_full} 
\begin{proof}
Here, we denote the client $k$ model as $ \boldsymbol{\theta}_{., k}$ and its $i$-th neuron as $ \boldsymbol{\theta}_{i, k}$.
Then, we represent the weighted average of all client models as $\Bar{\boldsymbol{\theta}}^{m,\tau}=\sum_{k=1}^{K}  \frac{n_k}{n} \boldsymbol{\theta}_{., k}^{m,\tau}$ and their $i$-th neuron as $\Bar{\boldsymbol{\theta}_{i}}^{m,\tau}=\sum_{k=1}^{K}  \frac{n_k}{n} \boldsymbol{\theta}_{i,k}^{m,\tau}$.
Meanwhile, for given  a data sample $(\boldsymbol{x}_{ k}^{m,\tau},\boldsymbol{y}_{ k}^{m,\tau})$, we formulate the model output as $\Bar{y}_k^{m,\tau} = f_{\Bar{\boldsymbol{\theta}}^{m,\tau}}(\boldsymbol{x}_k^{m,\tau})$ and
$\hat{y}_k^{m,\tau} = f_{{\boldsymbol{\theta}}_{., k}^{m,\tau}}(\boldsymbol{x}_{ k}^{m,\tau})$, and the gradient deviation as $\Delta \boldsymbol{g}_{i,k}^{m,\tau} = \nabla_{\boldsymbol{\theta}_{i,k}} \sigma(\boldsymbol{x}_k^{m,\tau} ; \boldsymbol{\theta}_{i,k}^{m,\tau}) - \nabla_{\boldsymbol{\theta}_i} \sigma(\boldsymbol{x}_k^{m,\tau} ; \Bar{\boldsymbol{\theta}_{i}}^{m,\tau}) $.

With these notations, we formulate the update of the  global model at the $m$-th round as: 
\begin{equation}  
\begin{aligned}
& \boldsymbol{\theta}_i^{m+1,0} - \boldsymbol{\theta}_i^{m,0}\\
= &   2\sum_{k=1}^{K} \frac{n_k}{n}\sum_{\tau=0}^{p} s^{m,\tau}(y_k^{m,\tau}-\hat{y}_k^{m,\tau}) \nabla_{\boldsymbol{\theta}_{i,k}} \sigma(\boldsymbol{x}_k^{m,\tau} ; \boldsymbol{\theta}_{i,k}^{m,\tau}) \\
= &   2   \sum_{\tau=0}^{p} {s}^{m,\tau} 
 \sum_{k=1}^{K}  \frac{n_k}{n} (y_k^{m,\tau}-\hat{y}_k^{m,\tau}) \left[\nabla_{\boldsymbol{\theta}_i} \sigma(\boldsymbol{x}_k^{m,\tau} ; \Bar{\boldsymbol{\theta}_{i}}^{m,\tau})\right. \\
 & +  \left. \left(\nabla_{\boldsymbol{\theta}_{i,k}} \sigma(\boldsymbol{x}_k^{m,\tau} ; \boldsymbol{\theta}_{i,k}^{m,\tau}) - \nabla_{\boldsymbol{\theta}_i} \sigma(\boldsymbol{x}_k^{m,\tau} ; \Bar{\boldsymbol{\theta}_{i}}^{m,\tau}) \right)  \right] \\
 = &  2   \sum_{\tau=0}^{p} {s}^{m,\tau} 
 \sum_{k=1}^{K}  \frac{n_k}{n} (y_k^{m,\tau}-\hat{y}_k^{m,\tau}) \left[\nabla_{\boldsymbol{\theta}_i} \sigma(\boldsymbol{x}_k^{m,\tau} ; \Bar{\boldsymbol{\theta}_{i}}^{m,\tau})\right. \\
& +  \left. \Delta \boldsymbol{g}_{i,k}^{m,\tau} \right]  \\
 = &   2   \sum_{\tau=0}^{p} {s}^{m,\tau} 
 \sum_{k=1}^{K}  \frac{n_k}{n} \left[ (y_k^{m,\tau}-\Bar{y}_k^{m,\tau})  \nabla_{\boldsymbol{\theta}_i} \sigma(\boldsymbol{x}_k^{m,\tau} ; \Bar{\boldsymbol{\theta}_{i}}^{m,\tau})\right. \\
&  + \left.   (\Bar{y}_k^{m,\tau}-\hat{y}_k^{m,\tau}) \nabla_{\boldsymbol{\theta}_i} \sigma(\boldsymbol{x}_k^{m,\tau} ; \Bar{\boldsymbol{\theta}_{i}}^{m,\tau}) \right] \\
& + 2   \sum_{\tau=0}^{p} {s}^{m,\tau}  \sum_{k=1}^{K}  \frac{n_k}{n} (y_k^{m,\tau}-\hat{y}_k^{m,\tau}) \Delta \boldsymbol{g}_{i,k}^{m,\tau}   \\
 = &    2   \sum_{\tau=0}^{p} {s}^{m,\tau} 
 \sum_{k=1}^{K}  \frac{n_k}{n}  (y_k^{m,\tau}-\Bar{y}_k^{m,\tau})  \nabla_{\boldsymbol{\theta}_i} \sigma(\boldsymbol{x}_k^{m,\tau} ; \Bar{\boldsymbol{\theta}_{i}}^{m,\tau})  \\
 \end{aligned}
\end{equation}
\begin{equation}  
\begin{aligned}
&   +  2   \sum_{\tau=0}^{p} {s}^{m,\tau}  \sum_{k=1}^{K}  \frac{n_k}{n}  \left[  (\Bar{y}_k^{m,\tau}-\hat{y}_k^{m,\tau}) \nabla_{\boldsymbol{\theta}_i} \sigma(\boldsymbol{x}_k^{m,\tau} ; \Bar{\boldsymbol{\theta}_{i}}^{m,\tau}) \right. \\
&   \quad + (y_k^{m,\tau}-\hat{y}_k^{m,\tau}) \Delta \boldsymbol{g}_{i,k}^{m,\tau} ]  \\
 = &    2   \sum_{\tau=0}^{p} {s}^{m,\tau} 
   (y^{m,\tau}-\Bar{y}^{m,\tau})  \nabla_{\boldsymbol{\theta}_i} \sigma(\boldsymbol{x}^{m,\tau} ; \Bar{\boldsymbol{\theta}_{i}}^{m,\tau})  \\
&  +      \sum_{\tau=0}^{p} { 2 {s}^{m,\tau}}{\boldsymbol{n}}_{i}^{m,\tau},
\end{aligned}
\end{equation}
where the last equality hold because of the assumption of one-pass data
processing and all client samples belong to the global distribution $\mathbb{P}$, i.e., $(\boldsymbol{x}_k, y_k) \in \mathbb{P}$. Here,   we represent the noise as:
\begin{equation}  
\begin{aligned}
  \boldsymbol{n}_{i}^{m,\tau} =   \sum_{k=1}^{K}  \frac{n_k}{n} &
 [(\Bar{y}_k^{m,\tau}-\hat{y}_k^{m,\tau}) \nabla_{\boldsymbol{\theta}_i} \sigma(\boldsymbol{x}_k^{m,\tau} ; \Bar{\boldsymbol{\theta}_{i}}^{m,\tau}) \\
& +   (y_k^{m,\tau}-\hat{y}_k^{m,\tau}) \Delta \boldsymbol{g}_{i,k}^{m,\tau}] \\
  =   \sum_{k=1}^{K}  \frac{n_k}{n} &
 [(\Bar{y}_k^{m,\tau}-\hat{y}_k^{m,\tau}) \nabla_{\boldsymbol{\theta}_i} \sigma(\boldsymbol{x}_k^{m,\tau} ; \Bar{\boldsymbol{\theta}_{i}}^{m,\tau}) \\
& +   (y_k^{m,\tau}-\hat{y}_k^{m,\tau}) (\nabla_{\boldsymbol{\theta}_{i,k}} \sigma(\boldsymbol{x}_k^{m,\tau} ; \boldsymbol{\theta}_{i,k}^{m,\tau}) \\
& \quad\quad\quad - \nabla_{\boldsymbol{\theta}_i} \sigma(\boldsymbol{x}_k^{m,\tau} ; \Bar{\boldsymbol{\theta}_{i}}^{m,\tau}))] .
 \end{aligned}
\end{equation}
Therefore, we obtain (\ref{equation:update_FedAvg_simple}), i.e.,
\begin{equation} 
\begin{aligned}
\boldsymbol{\theta}_i^{m+1,\tau+1}
 = & \boldsymbol{\theta}_i^{m,\tau} + 2  {s}^{m,\tau} 
   (y^{m,\tau}-\Bar{y}^{m,\tau})  \nabla_{\Bar{\boldsymbol{\theta}}_i} \sigma(\boldsymbol{x}^{m,\tau} ; \Bar{\boldsymbol{\theta}_{i}}^{m,\tau})  \\
&  +       2 {s}^{m,\tau} \boldsymbol{n}_{i}^{m,\tau}.  
\end{aligned}
\end{equation}

 Furthermore, according to   Definition \ref{def:data_heterogeneity},  $ \boldsymbol{n}_{i}^{m,\tau}$ can be regarded as the noise induced by data heterogeneity in FL. 
Since $\boldsymbol{\theta}_i^{m,p-1} = \Bar{\boldsymbol{\theta}_{i}}^{m,p-1}=\sum_{k=1}^{K}  \frac{n_k}{n} \boldsymbol{\theta}_{., k}^{m,p-1}$, we take Assumption \ref{assumption5} and simplify (\ref{equation:update_FedAvg_simple}) with $\tau \in [0, Mp-1]$ when given a total training round $M$ as:
\begin{equation} 
\begin{aligned}
\boldsymbol{\theta}_i^{\tau+1}=  &  \boldsymbol{\theta}_i^{\tau}+2 s_\tau\left(y^\tau-\Bar{y}^{\tau}\right) \nabla_{\boldsymbol{\theta}_i} \sigma\left(\boldsymbol{x}_\tau ; \boldsymbol{\theta}_i^\tau\right)   \\ &  +2 {s}^{\tau}  { \boldsymbol{n}}_i^\tau. 
\end{aligned}
\end{equation}
Therefore, we obtain (\ref{equation:update_FedAvg_onestep}).

    
\end{proof}

\end{document}